\documentclass[11pt,twocolumn]{article}
\usepackage{amsthm}

\newtheorem{definition}{Definition}
\newtheorem{lemma}{Lemma}
\newtheorem{theorem}{Theorem}
\newtheorem{corollary}{Corollary}
\usepackage{flushend}

\usepackage[T1]{fontenc}
\usepackage[applemac]{inputenc}
\usepackage[english]{babel}

\usepackage{amsmath}
\usepackage{amsxtra}
\usepackage{amstext}
\usepackage{amsfonts}
\usepackage{amssymb}

\usepackage{mathrsfs}
\usepackage{dsfont}

\usepackage{algorithm}
\usepackage{algorithmic}

\usepackage{graphicx}
\usepackage{subfigure}
\usepackage{sidecap}

\usepackage{ctable}
\usepackage{datetime}
\usepackage{enumerate}
\usepackage{natbib}

\graphicspath{{./Images/}}

\newtheorem{assumption}{Assumption}

\newcommand{\A}{\mathcal{A}}

\newcommand{\F}{\mathcal{F}}
\renewcommand{\H}{\mathcal{H}}
\newcommand{\M}{\mathcal{M}}
\newcommand{\R}{\mathcal{R}}
\newcommand{\X}{\mathcal{X}}

\renewcommand{\P}{\mathsf{P}}

\newcommand{\eps}{\varepsilon}

\renewcommand{\vec}[1]{\mathbf{#1}}

\newcommand{\EE}[2][]{\mathbb{E}_{#1}\left[#2\right]}
\newcommand{\PP}[2][]{\mathbb{P}_{#1}\left[#2\right]}
\newcommand{\II}{\mathbb{I}}

\DeclareMathOperator*{\argmax}{\rm argmax}
\DeclareMathOperator*{\argmin}{\rm argmin}

\newcommand{\hbeta}{\hat{\beta}}
\newcommand{\hgamma}{\hat{\gamma}}
\newcommand{\hsigma}{\hat{\sigma}}

\newcommand{\rt}{r_{\mathrm{target}}}
\newcommand{\x}{[x]}

\newcommand{\ah}{a^*}

\newcommand{\betah}[1][\vec{h}]{\beta_{#1}}
\newcommand{\gammah}[1][\vec{h}]{\gamma_{#1}}

\providecommand{\set}[1]{\left\{#1\right\}}
\providecommand{\abs}[1]{\left\lvert#1\right\rvert}

\newcommand{\MDP}[1][r]{(\X,\A,\P,#1,\gamma)}

\newcommand{\ie}{\textit{i.e.,~}}

\begin{document}

\title{Multi-class Generalized Binary Search for Active Inverse Reinforcement Learning}


\author{Francisco S.\ Melo\\INESC-ID/Instituto Superior T\'ecnico\\Portugal\\  \textit{fmelo@inesc-id.pt} \and Manuel Lopes\\INRIA Bordeaux Sud-Ouest\\France\\
           \textit{manuel.lopes@inria.fr}}

\date{}

\maketitle


\begin{abstract} 

This paper addresses the problem of learning a task from demonstration. We adopt the framework of \emph{inverse reinforcement learning}, where tasks are represented in the form of a reward function. Our contribution is a novel active learning algorithm that enables the learning agent to query the expert for more informative demonstrations, thus leading to more sample-efficient learning. For this novel algorithm (Generalized Binary Search for Inverse Reinforcement Learning, or GBS-IRL), we provide a theoretical bound on sample complexity and illustrate its applicability on several different tasks. To our knowledge, GBS-IRL is the first active IRL algorithm with provable sample complexity bounds. We also discuss our method in light of other existing methods in the literature and its general applicability in multi-class classification problems. Finally, motivated by recent work on learning from demonstration in robots, we also discuss how different forms of human feedback can be integrated in a transparent manner in our learning framework.

\end{abstract}


\section{Introduction}%
\label{Sec:Introduction}




\emph{Social learning}, where an agent uses information provided by other individuals to polish or acquire anew skills, is likely to become one primary form of programming such complex intelligent systems \citep{schaal99tcs}. Paralleling the social learning ability of human infants, an artificial system can retrieve a large amount of task related information by observing and/or interacting with other agents engaged in relevant activities. For example, the behavior of an expert can bias an agent's exploration of the environment, improve its knowledge of the world, or even lead it to reproduce parts of the observed behavior \citep{melo07aisb}.


In this paper we are particularly interested in \emph{learning from demonstration}. This particular form of social learning is commonly associated with \emph{imitation} and \emph{emulation} behaviors in nature \citep{lopes09ab}. It is also possible to find numerous successful examples of robot systems that learn from demonstration \citep[see the survey works of][]{argall09ras, lopes10chap}. In the simplest form of interaction, the demonstration may consist of examples of the right action to take in different situations. 

In our approach to learning from demonstration we adopt the formalism of \emph{inverse reinforcement learning} (IRL), where the task is represented as a \emph{reward function} \citep{ng00icml}. From this representation, the agent can then construct its own policy and solve the target task. However, and unlike many systems that learn from demonstration, in this paper we propose to combine ideas from \emph{active learning} \citep{settles09tr} with IRL, in order to reduce the data requirements during learning. In fact, many agents able to learn from demonstration are designed to process batches of data, typically acquired before any actual learning takes place. Such data acquisition process fails to take advantage of any information the learner may acquire in early stages of learning to guide the acquisition of new data. Several recent works have proposed that a more \emph{interactive} learning may actually lead to improved learning performance. 


We adopt a Bayesian approach to IRL, following \cite{ramachandran07ijcai}, and allow the learning agent to actively select and query the expert for the desired behavior at the most informative situations. We contribute a theoretical analysis of our algorithm that provides a bound on the sample complexity of our learning approach and illustrate our method in several problems from the IRL literature. 


Finally, even if learning from demonstration is the main focus of our paper and an important skill for intelligent agents interacting with human users, the ability to accommodate different forms of feedback is also useful. In fact, there are situations where the user may be unable to properly demonstrate the intended behavior and, instead, prefers to describe a task in terms of a \emph{reward function}, as is customary in reinforcement learning \citep{sutton98}. As an example, suppose that the user wants the agent to learn how to navigate a complex maze. The user may experience difficulties in navigating the maze herself and may, instead, allow the agent to explore the maze and reward it for exiting the maze. 

Additionally, recent studies on the behavior of na\"{i}ve users when instructing agents (namely, robots) showed that the feedback provided by humans is often ambiguous and does not map in any obvious manner to either a reward function or a policy \citep{thomaz08aij,cakmak10icdl}. For instance, it was observed that human users tend to provide learning agents with anticipatory or guidance rewards, a situation seldom considered in reinforcement learning \citep{thomaz08aij}. This study concludes that robust agents able to successfully learn from human users should be flexible to accommodate different forms of feedback from the user.


In order to address the issues above, we discuss how other forms of expert feedback (beyond policy information) may be integrated in a seamless manner in our IRL framework, so that the learner is able to recover efficiently the target task. In particular, we show how to combine both \emph{policy and reward information} in our learning algorithm. Our approach thus provides a useful bridge between reinforcement learning (or learning by trial and error) and imitation learning (or learning from demonstration), a line of work seldom explored in the literature \citep[see, however, the works of][and discussion in Section~\ref{Subsec:Related}]{knox10aamas,knox11ijcai}.


The paper is organized as follows. In the remainder of this section, we provide an overview of related work on social learning, particularly on learning from demonstration. We also discuss relevant research in IRL and active learning, and discuss our contributions in light of existing work. Section~\ref{Sec:Background} revisits core background concepts, introducing the notation used throughout the paper. Section~\ref{Sec:ActiveIRL} introduces our active IRL algorithm, GBS-IRL, and provides a theoretical analysis of its sample complexity. Section~\ref{Sec:Experiments} illustrates the application of GBS-IRL in several problems of different complexity, providing an empirical comparison with other methods in the literature. Finally, Section~\ref{Sec:Conclusions} concludes the paper, discussing directions for future research.


\subsection{Related Work}%
\label{Subsec:Related}


There is extensive literature reporting research on intelligent agents that learn from expert advice. Many examples feature robotic agents that learn simple tasks from different forms of human feedback. Examples include the robot {\sf Leonardo} that is able to learn new tasks by observing changes induced in the world (as perceived by the robot) by a human demonstrating the target task \cite{breazeal04ijhr}. During learning, {\sf Leonardo} provides additional feedback on its current understanding of the task that the human user can then use to provide additional information. We refer the survey works of \cite{argall09ras,lopes10chap} for a comprehensive discussion on learning from demonstration.



In this paper, as already mentioned, we adopt the inverse reinforcement learning (IRL) formalism introduced in the seminal paper by \cite{ng00icml}. One appealing aspect of the IRL approach to learning from demonstration is that the learner is not just ``mimicking'' the observed actions. Instead, the learner infers the purpose behind the observed behavior and sets such purpose as its goal. IRL also enables the learner to accommodate for differences between itself and the demonstrator \citep{lopes09ab}. 

The appealing features discussed above have led several researchers to address learning from demonstration from an IRL perspective. \cite{abbeel04icml} explored inverse reinforcement learning in a context of \emph{apprenticeship learning}, where the purpose of the learning agent is to replicate the behavior of the demonstrator, but is only able to observe a sequence of states experienced during task execution. The IRL formalism allows the learner to reason about which tasks could lead the demonstrator to visit the observed states and infer how to replicate the inferred behavior. \citeauthor{syed08icml} \citep{syed08icml,syed08nips} have further explored this line of reasoning from a game-theoretic perspective, and proposed algorithms to learn from demonstration with provable guarantees on the performance of the learner.

\cite{ramachandran07ijcai} introduced \emph{Bayesian inverse reinforcement learning} (BIRL), where the IRL problem is cast as a Bayesian inference problem. Given a prior distribution over possible target tasks, the algorithm uses the demonstration by the expert as evidence to compute the posterior distribution over tasks and identify the target task. Unfortunately, the Monte-Carlo Markov chain (MCMC) algorithm used to approximate the posterior distribution is computationally expensive, as it requires extensive sampling of the space of possible rewards. To avoid such complexity, several posterior works have departed from the BIRL formulation and instead determine the task that maximizes the likelihood of the observed demonstration \citep{lopes09ecml,babes11icml}. 

The aforementioned maximum likelihood approaches of \cite{lopes09ecml} and \cite{babes11icml} take advantage of the underlying IRL problem structure and derive simple gradient-based algorithms to determine the maximum likelihood task representation. Two closely related works are the \emph{maximum entropy approach} of \cite{ziebart08aaai} and the \emph{gradient IRL approach} of \cite{neu07uai}. While the former selects the task representation that maximizes the likelihood of the observed expert behavior, under the maximum entropy distribution, the latter explores a gradient-based approach to IRL, but the where the task representation is selected so as to induce a behavior as similar as possible to the expert behavior.


Finally, \cite{ross10aistats} propose a learning algorithm that reduces imitation learning to a classification problem. The classifier prescribes the best action to take in each possible situation that the learner can encounter, and is successively improved by enriching the data-set used to train the classifier.


All above works are designed to learn from whatever data is available to them at learning time, data that is typically acquired before any actual learning takes place. Such data acquisition process fails to take advantage of the information that the learner acquires in early stages of learning to guide the acquisition of new, more informative data. \emph{Active learning} aims to reduce the data requirements of learning algorithms by actively selecting potentially informative samples, in contrast with random sampling from a predefined distribution \citep{settles09tr}. In the case of learning from demonstration, active learning can be used to reduce the number of situations that the expert/human user is required to demonstrate. Instead, the learner should proactively ask the expert to demonstrate the desired behavior at the most informative situations.

\emph{Confidence-based autonomy} (CBA), proposed by \cite{chernova09jair}, also enables a robot to learn a task from a human user by building a mapping between situations that the robot has encountered and the adequate actions. This work already incorporates a mechanism that enables the learner to ask the expert for the right action when it encounters a situation in which it is less confident about the correct behavior. The system also allows the human user to provide corrective feedback as the robot executes the learned task.%
\footnote{Related ideas are further explored in the \emph{dogged learning} architecture of \cite{grollman07icra}.}
The querying strategy in CBA can be classified both as \emph{stream-based} and as \emph{mellow} \citep[see discussions in the survey works of][]{settles09tr,dasgupta11tcs}. Stream-based, since the learner is presented with a stream of samples (in the case of CBA, samples correspond to possible situations) and only asks for the labels (\ie correct actions) of those samples it feels uncertain about. Mellow, since it does not seek highly informative samples, but queries any sample that is at all informative. 

In the IRL literature, active learning was first explored in a preliminary version of this paper \citep{lopes09ecml}. In this early version, the learner actively queries the expert for the correct action in those states where it is most uncertain about the correct behavior. Unlike CBA, this active sampling approach is \emph{aggressive} and uses \emph{membership query synthesis}. Aggressive, since it actively selects highly informative samples. And, unlike CBA, it can select (``synthesize'') queries from the whole input space. \cite{judah11icml} propose a very similar approach, the \emph{imitation query-by-committee} (IQBC) algorithm, which differs only from the previous active sampling approach in the fact that the learner is able to accommodate the notion of ``bad states'', \ie states to be avoided during task execution.

\cite{cohn11aaai} propose another closely related approach that, however, uses a different criterion to select which situations to query. EMG-AQS (Expected Myopic Gain Action Querying Strategy) queries the expert for the correct action in those states where the expected \emph{gain of information} is potentially larger. Unfortunately, as discussed by \cite{cohn11aaai}, the determination of the expected gain of information requires extensive computation, rendering EMG-AQS computationally costly. On a different line of work, \cite{ross11aistats,judah12uai} address imitation learning using a no-regret framework, and propose algorithms for direct imitation learning with provable bounds on the regret. Finally, \cite{melo10ecml} use active learning in a metric approach to learning from demonstration.

Our approach in this paper is a modified version of our original active sampling algorithm \citep{lopes09ecml}. We depart from the generalized binary search (GBS) algorithm of \cite{nowak11tit} and adapt it to the IRL setting. To this purpose, we cast IRL as a (multi-class) classification problem and extend the GBS algorithm of \cite{nowak11tit} to this multi-class setting. We analyze the sample complexity of our GBS-IRL approach, thus providing the first active IRL algorithm with provable bounds on sample complexity. Also, to the extent of our knowledge, GBS-IRL is the first aggressive active learning algorithm for non-separable, multi-class data \citep{dasgupta11tcs}.


We conclude this discussion of related work by pointing out that all above works describe systems that learn from human feedback. However, other forms of expert advice have also been explored in the agent learning literature. \cite{price99icml,price03jair} have explored how a learning agent can improve its performance by observing other similar agents, in what could be seen as ``implicit'' imitation learning. In these works, the demonstrator is, for all purposes, oblivious to the fact that its actions are being observed and learned from. Instead, the learned observes the behavior of the other agents and extracts information that may be useful for its own learning (for example, it may extract useful information about the world dynamics). 

In a more general setting, \cite{barto04adp} discuss how different forms of supervisory information can be integrated in a reinforcement learning architecture to improve learning. Finally, \cite{knox09kcap,knox10aamas} introduce the {\sc tamer} paradigm, that enables a reinforcement learning agent to use human feedback (in addition to its reinforcement signal) to guide its learning process.




\subsection{Contributions}%
\label{Subsec:Contributions}

Our contributions can be summarized as follows:
\begin{itemize}
\item A \emph{novel active IRL algorithm}, GBS-IRL, that extends generalized binary search to a multi-class setting in the context of IRL. 
\item The \emph{sample-complexity analysis of GBS-IRL}. We establish, under suitable conditions, the exponential convergence of our active learning method, as a function of the number of samples. As pointed out earlier, to our knowledge ours is the first work providing sample complexity bounds on active IRL. Several experimental results confirm the good sample performance of our approach.  
\item A general discussion on how different forms of expert information (namely action and reward information) can be integrated in our IRL setting. We illustrate the applicability of our ideas in several simple scenarios and discuss the applicability of these different sources of information in face of our empirical results.
\end{itemize}

From a broader perspective, our analysis is a non-trivial extension of the results of \cite{nowak11tit} to a multiclass setting, having applications not only on IRL but on any multiclass classification problem.


\section{Background and Notation}%
\label{Sec:Background}

This section introduces background material on Markov decision processes and the Bayesian inverse reinforcement learning formalism, upon which our contributions are developed.


\subsection{Markov Decision Processes}%
\label{Subsec:MDPs}

A \emph{Markov decision problem} (MDP) describes a sequential decision problem in which an agent must choose the sequence of actions that maximizes some reward-based optimization criterion. Formally, an MDP $\M$ is a tuple $\M=\MDP$, where $\X$ represents the state-space, $\A$ the finite action space, $\P$ represents the transition probabilities, $r$ is the reward function and $\gamma$ is a positive discount factor. $\P(y\mid x,a)$ denotes the probability of transitioning from state $x$ to state $y$ when action $a$ is taken, \ie
\begin{displaymath}
\P(y\mid x,a)=\PP{X_{t+1}=y\mid X_t=x,A_t=a},
\end{displaymath}
where each $X_t, t=1,\ldots$, is a random variable (r.v.) demoting the state of the process at time-step $t$ and $A_t$ is a r.v.\ denoting the action of the agent at time-step $t$. 

%
%

A \emph{policy} is a mapping $\pi:\X\times\A\to[0,1]$, where $\pi(x,a)$ is the probability of choosing action $a\in\A$ in state $x\in\X$. Formally,
\begin{displaymath}
\pi(x,a)=\PP{A_t=a\mid X_t=x}.
\end{displaymath}
It is possible to associate with any such policy $\pi$ a \emph{value-function},
\begin{displaymath}
V^\pi(x)=\EE[\pi]{\sum_{t=0}^\infty\gamma^tr(X_t,A_t)\mid{X}_0=x},
\end{displaymath}
where the expectation is now taken over possible trajectories of $\set{X_t}$ induced by policy $\pi$. The purpose of the agent is then to select a policy $\pi^*$ such that
\begin{displaymath}
V^{\pi^*}(x)\geq{V}^\pi(x),
\end{displaymath}
for all $x\in\X$. Any such policy is an \emph{optimal policy} for that MDP and the corresponding value function is denoted by $V^*$.

Given any policy $\pi$, the following recursion holds
\begin{displaymath}
V^\pi(x)=r_\pi(x)+\gamma\sum_{y\in\X}\P_\pi(x,y)V^\pi(y)
\end{displaymath}
where $\P_\pi(x,y)=\sum_{a\in\A}\pi(x,a)\P_a(x,y)$ and $r_\pi(x)=\sum_{a\in\A}\pi(x,a)r(x,a)$. For the particular case of the optimal policy $\pi^*$, the above recursion becomes
\begin{displaymath}
V^*(x)=\max_{a\in\A}\left[r(x,a)+\gamma\sum_{y\in\X}\P_a(x,y)V^*(y)\right].
\end{displaymath}

We also define the $Q$-function associated with a policy $\pi$ as
\begin{displaymath}
Q^\pi(x,a)=r(x,a)+\gamma\sum_{y\in\X}\P_a(x,y)V^\pi(y)
\end{displaymath}
which, in the case of the optimal policy, becomes
\begin{equation}\label{Eq:Q-function}
\begin{split}
  Q^*(x,a)
    &=r(x,a)+\gamma\sum_{y\in\X}\P_a(x,y)V^*(y)\\
    &=r(x,a)+\gamma\sum_{y\in\X}\P_a(x,y)\max_{b\in\A}Q^*(y,b).
\end{split}
\end{equation}


\subsection{Bayesian Inverse Reinforcement Learning}%
\label{Subsec:BIRL}

As seen above, an MDP describes a sequential decision making problem in which an agent must choose its actions so as to maximize the total discounted reward. In this sense, the reward function in an MDP encodes the \emph{task} of the agent.

\emph{Inverse reinforcement learning} (IRL) deals with the problem of recovering the task representation (\ie the reward function) given a demonstration of the behavior to be learned (\ie the desired policy). In this paper we adopt the formulation in \cite{ramachandran07ijcai}, where IRL is cast as a \emph{Bayesian inference problem}, in which the agent is provided with samples of the desired policy, $\pi^*$, and it must identify the target reward function, $r^*$, from a general set of possible functions $\R$. Prior to the observation of any policy sample and given any measurable set $R\subset\R$, the initial belief that $r^*\in R$ is encoded in the form of a probability density function $\rho$ defined on $\R$, \ie
\begin{displaymath}
\PP{r^*\in R}=\int_R\rho(r)dr.
\end{displaymath}
As discussed by \cite{ramachandran07ijcai,lopes09ecml}, it is generally impractical to explicitly maintain and update $\rho$. Instead, as in the aforementioned works, we work with a finite (but potentially very large) sample of $\R$ obtained according to $\rho$. We denote this sample by $\R_\rho$, and associate with each element $r_k\in\R_\rho$ a \emph{prior probability} $p_0(r_k)$ given by
\begin{displaymath}
p_0(r_k)=\frac{\rho(r_k)}{\sum_i\rho(r_i)}.
\end{displaymath}
Associated with each reward $r_k\in\R_\rho$ and each $x\in\X$, we define the \emph{set of greedy actions at $x$} with respect to $r_k$ as 
\begin{displaymath}
\A_k(x)=\set{a\in\A\mid a\in\argmax Q_k(x,a)}
\end{displaymath}
where $Q_k$ is the $Q$-function associated with the optimal policy for $r_k$, as defined in \eqref{Eq:Q-function}.
From the sets $\A_k(x), x\in\X$, we define the \emph{greedy policy} with respect to $r_k$ as the mapping $\pi_k:\X\times\A\to[0,1]$ given by
\begin{displaymath}
\pi_k(x,a)=\frac{\II_{\A_k(x)}(a)}{\abs{\A_k(x)}},
\end{displaymath}
where we write $\II_U$ to denote the indicator function for a set $U$. In other words, for each $x\in\X$, the greedy policy with respect to $r_k$ is defined as a probability distribution that is uniform in $\A_k(x)$ and zero in its complement. We assume, without loss of generality, that for any $r_i,r_j\in\R_\rho$, $\A_i(x)\neq\A_j(x)$ for at least one $x\in\X$.%
\footnote{This assumption merely ensures that there are no redundant rewards on $\R_\rho$. If two such rewards $r_i,r_j$ existed in $\R_\rho$, we could safely discard one of the two, say $r_j$, setting $p_0(r_i)\leftarrow p_0(r_i)+p_0(r_j)$.}

For any $r_k\in\R_\rho$, consider a perturbed version of $\pi_k$ where, for each $x\in\X$, action $a\in\A$ is selected with a probability
\begin{equation}\label{Eq:eGreedy}
\hat{\pi}_k(x,a)=
\begin{cases}
\beta_k(x) & \text{if $a\notin\A_k(x)$}\\
\gamma_k(x) & \text{if $a\in\A_k(x)$},
\end{cases}
\end{equation}
where, typically, $\beta_k(x)<\gamma_k(x)$.%
\footnote{Policy $\hat{\pi}_k$ assigns the same probability, $\gamma_k(x)$ to all actions that, for the particular reward $r_k$, are optimal in state $x$. Similarly, it assigns the same probability, $\beta_k(x)$, to all corresponding sub-optimal actions. This perturbed version of $\pi_k$ is convenient both for its simplicity and because it facilitates our analysis. However, other versions of perturbed policies have been considered in the IRL literature---see, for example, the works of \cite{ramachandran07ijcai,neu07uai,lopes09ecml}.}
We note that both $\pi_k$ and the uniform policy can be obtained as limits of $\hat{\pi}_k$, by setting $\beta_k(x)=0$ or $\beta_k(x)=\gamma_k(x)$, respectively. Following the Bayesian IRL paradigm, the likelihood of observing an action $a$ by the demonstrator at state $x$, given that the target task is $r_k$, is now given by
\begin{equation}\label{Eq:Likelihood}
\ell_k(x,a)=\PP{A_t=a\mid X_t=x, r^*=r_k}=\hat{\pi}_k(x,a).
\end{equation}

Given a history of $t$ (independent) observations, $\F_t=\set{(x_\tau,a_\tau),\tau=0,\ldots,t}$, the likelihood in \eqref{Eq:Likelihood} can now be used in a standard Bayesian update to compute, for every $r_k\in\R_\rho$, the posterior probability
\begin{align*}
p_t(r_k)
  &=\frac{\PP{r^*=r_k}\PP{\F_t\mid r^*=r_k}}{Z}\\
  &=\frac{p_0(r_k)\prod_{\tau=0}^t\ell_k(x_\tau,a_\tau)}{Z},
\end{align*}
where $Z$ is a normalization constant.

For the particular case of $r^*$ we write the corresponding perturbed policy as 
\begin{displaymath}
\hat{\pi}^*(x,a)=
\begin{cases}
\beta^*(x) & \text{if $a\notin\A^*(x)$}\\
\gamma^*(x) & \text{if $a\in\A^*(x)$},
\end{cases}
\end{displaymath}
and denote the \emph{maximum noise level} as the positive constant $\alpha$ defined as
\begin{displaymath}
\alpha=\sup_{x\in\X}\beta^*(x).
\end{displaymath}


\section{Multiclass Active Learning for Inverse Reinforcement Learning}%
\label{Sec:ActiveIRL}

In this section we introduce our active learning approach to IRL.


\subsection{Preliminaries}

To develop an active learning algorithm for this setting, we convert the problem of determining $r^*$ into an equivalent classification problem. This mostly amounts to rewriting of the Bayesian IRL problem from Section~\ref{Sec:Background} using a different notation. 

We define the hypothesis space $\H$ as follows. For every $r_k\in\R_\rho$, the $k$th hypothesis $\vec{h}_k:\X\to\set{-1,1}^{\abs{\A}}$ is defined as the function
\begin{displaymath}
h_k(x,a)=2\II_{\A_k(x)}(a)-1,
\end{displaymath}
where we write $h_k(x,a)$ to denote the $a$th component of $\vec{h}_k(x)$. Intuitively, $\vec{h}_k(x)$ identifies (with a value of $1$) the greedy actions in $x$ with respect to $r_k$, assigning a value of $-1$ to all other actions. We take $\H$ as the set of all such functions $\vec{h}_k$. Note that, since every reward prescribes at least one optimal action per state, it holds that for every $\vec{h}\in\H$ and every $x\in\X$ there is at least one $a\in\A$ such that $h(x,a)=1$. We write $\vec{h}^*$ to denote the target hypothesis, corresponding to $r^*$.

As before, given a hypothesis $\vec{h}\in\H$, we define the \emph{set of greedy actions at $x$} according to $\vec{h}$ as 
\begin{displaymath}
\A_\vec{h}(x)=\set{a\in\A\mid h(x,a)=1}.
\end{displaymath}
For an indexed set of samples, $\set{(x_\lambda,a_\lambda),\lambda\in\Lambda}$, we write $h_\lambda$ to denote $h(x_\lambda,a_\lambda)$, when the index set is clear from the context.

The prior distribution $p_0$ over $\R_\rho$ induces an equivalent distribution over $\H$, which we abusively also denote as $p_0$, and is such that $p_0(\vec{h}_k)=p_0(r_k)$. We let the history of observations up to time-step $t$ be 
\begin{displaymath}
\F_t=\set{(x_\tau,a_\tau),\tau=0,\ldots,t},
\end{displaymath} 
and $\betah$ and $\gammah$ be the estimates of $\beta^*$ and $\gamma^*$ associated with the hypothesis $\vec{h}$. Then, the distribution over $\H$ after observing $\F_t$ can be updated using Bayes rule as
\begin{align}
\nonumber%
p_t(\vec{h})
   &\triangleq\PP{\vec{h}^*=\vec{h}\mid\F_t}\\
\nonumber%
   &\propto\PP{a_t\mid x_t,\vec{h}^*=\vec{h},\F_{t-1}}\PP{\vec{h}^*=\vec{h}\mid\F_{t-1}}\\
\nonumber%
   &=\PP{a_t\mid x_t,\vec{h}^*=\vec{h}}\PP{\vec{h}=\vec{h}^*\mid\F_{t-1}}\\
\label{Eq:Bayes-first}%
   &\approx\gamma_{\vec{h}}(x_t)^{(1+h_t)/2}\beta_{\vec{h}}(x_t)^{(1-h_t)/2}p_{t-1}(\vec{h}),
\end{align}
where we assume, for all $x\in\X$,
\begin{equation}\label{Eq:Steps-basic} 
\abs{\A_{\vec{h}}(x)}\gamma_{\vec{h}}(x)\leq\abs{\A^*(x)}\gamma^*(x), 
\end{equation}
and $p_t(\vec{h})$ is normalized so that $\sum_{\vec{h}\in\H}p_t(\vec{h})=1$. Note that, in \eqref{Eq:Bayes-first}, we accommodate for the possibility of having access (for each hypothesis) to inaccurate estimates $\betah$ and $\gammah$ of $\beta^*$ and $\gamma^*$, respectively. 

We consider a partition of the state-space $\X$ into a disjoint family of $N$ sets, $\Xi=\set{\X_1,\ldots,\X_N}$ such that all hypotheses $\vec{h}\in\H$ are constant in each set $\X_i,i=1\ldots,N$. In other words, any two states $x,y$ lying in the same $\X_i$ are indistinguishable, since $h(x,a)=h(y,a)$ for all $a\in\A$ and all $\vec{h}\in\H$. This means that our hypothesis space $\H$ induces an equivalence relation in $\X$ in which two elements $x,y\in\X$ are equivalent if $\set{x,y}\subset\X_i$. We write $\x_i$ to denote the (any) representative of the set $\X_i$.%
\footnote{While this partition is, perhaps, of little relevance in problems with a small state-space $\X$, it is central in problems with large (or infinite) state-space, since the state to be queried has to be selected from a set of $N$ alternatives, instead of the (much larger) set of $\abs{\X}$ alternatives.}

The following definitions extend those of \cite{nowak11tit}.

\begin{definition}[$k$-neighborhood]
Two sets $\X_i,\X_j\in\Xi$ are said to be \emph{$k$-neighbors} if the set
\begin{displaymath}
\big\{\vec{h}\in\H\mid \A_\vec{h}(\x_i)\neq\A_\vec{h}(\x_j)\big\}
\end{displaymath}
has, at most, $k$ elements, \ie if there are $k$ or fewer hypotheses in $\H$ that output different optimal actions in $\X_i$ and $\X_j$.
\end{definition}

\begin{definition}
The pair $(\X,\H)$ is \emph{$k$-neighborly} if, for any two sets $\X_i,\X_j\in\Xi$, there is a sequence $\set{\X_{\ell_0},\ldots,\X_{\ell_n}}\subset\Xi$ such that
\begin{itemize}
\item $\X_{\ell_0}=\X_i$ and $\X_{\ell_n}=\X_j$;
\item For any $m$, $\X_{\ell_m}$ and $\X_{\ell_{m+1}}$ are $k$-neighbors.
\end{itemize}
\end{definition}

The notion of $k$-neighborhood structures the state-space $\X$ in terms of the hypotheses space $\H$, and this structure can be exploited for active learning purposes.


\subsection{Active IRL using GBS}%
\label{Subsec:SingleAction}

In defining our active IRL algorithm, we first consider a simplified setting in which the following assumption holds. We postpone to Section~\ref{Subsec:MultiAction} the discussion of the more general case.
\begin{assumption}\label{Ass:OneAction}
For every $h\in\H$ and every $x\in\X$, $\abs{\A_\vec{h}(x)}=1$. 
\end{assumption}
In other words, we focus on the case where all hypothesis considered prescribe a unique optimal action per state. A single optimal action per state implies that the noise model can be simplified. In particular, the noise model can now be constant across hypothesis, since all $\vec{h}\in\H$ prescribes the same number of optimal actions in each state (namely, one). We denote by $\hgamma(x)$ and $\hbeta(x)$ the estimates of $\gamma^*$ and $\beta^*$, respectively, and consider a Bayesian update of the form:
\begin{align}
\label{Eq:Bayes}   
p_t(\vec{h})
   &\propto \frac{1}{Z}\hgamma(x_t)^{(1+h_t)/2}\hbeta(x_t)^{(1-h_t)/2}p_{t-1}(\vec{h}),
\end{align} 
with $1-\hgamma(x)=(\abs{\A}-1)\hbeta(x)$ and $Z$ an adequate normalization constant. For this simpler case, \eqref{Eq:Steps-basic} becomes
\begin{align}\label{Eq:Steps} 
\hbeta(x)\geq\beta^*(x) && \text{and} && \hgamma(x)&\leq\gamma^*(x),
\end{align}
where, as before, we overestimate the noise rate $\beta^*(x)$. For a given probability distribution $p$, define the \emph{weighted prediction in $x$} as
\begin{displaymath}
W(p,x)=\max_{a\in\A}\sum_{\vec{h}\in\H}p(\vec{h})h(x,a),
\end{displaymath}
and the \emph{predicted action at $x$} as
\begin{displaymath}
A^*(p,x)=\argmax_{a\in\A}\sum_{\vec{h}\in\H}p(\vec{h})h(x,a).
\end{displaymath}
We are now in position to introduce a first version of our active learning algorithm for inverse reinforcement learning, that we dub \emph{Generalized Binary Search for IRL} (GBS-IRL). GBS-IRL is summarized in Algorithm~\ref{Alg:ActIRL}.
\begin{algorithm}[!tb]
\footnotesize
\caption{GBS-IRL (version~1)}\label{Alg:ActIRL}
  \begin{algorithmic}[1]
    \REQUIRE MDP parameters $\M\backslash{r}$
    \REQUIRE Reward space $\R_\rho$
    \REQUIRE Prior distribution $p_0$ over $\R$
    \STATE Compute $\H$ from $\R_\rho$
    \STATE Determine partition $\Xi={\X_1,\ldots\X_N}$ of $\X$
    \STATE Set $\F_0=\emptyset$
    \FORALL{$t=0,\ldots$}
    \STATE Set $c_t=\min_{i=1,\ldots,N}W(p_t,\x_i)$
    \IF{there are $1$-neighbor sets $\X_i,\X_j$ such that 
    \begin{align*}
      W(p_t,\x_i)&>c_t,&
      W(p_t,\x_j)&>c_t\\
      A^*(p_t,\x_i)&\neq A^*(p_t,\x_j),
    \end{align*}}
    \STATE Sample $x_{t+1}$ from $\X_i$ or $\X_j$ with probability $1/2$
    \ELSE
    \STATE Sample $x_{t+1}$ from the set $\X_i$ that minimizes $W(p_t,\x_i)$.
    \ENDIF
    \STATE Obtain noisy response $a_{t+1}$
	\STATE Set $\F_{t+1}\leftarrow\F_t\cup\set{(x_{t+1},a_{t+1})}$
    \STATE Update $p_{t+1}$ from $p_t$ using \eqref{Eq:Bayes}
    \ENDFOR
    \RETURN $\hat{\vec{h}}_t=\argmax_{\vec{h}\in\H}p_t(\vec{h})$.
  \end{algorithmic}
\end{algorithm}
This first version of the algorithm relies critically on Assumption~\ref{Ass:OneAction}. In Section~\ref{Subsec:MultiAction}, we discuss how Algorithm~\ref{Alg:ActIRL} can be modified to accommodate situations in which Assumption~\ref{Ass:OneAction} does not hold.

Our analysis of GBS-IRL relies on the following fundamental lemma that generalizes Lemma~3 of \cite{nowak11tit} to multi-class settings.

\begin{lemma}\label{Lemma:Incoherence}
Let $\H$ denote a hypothesis space defined over a set $\X$, where $(\X,\H)$ is assumed $k$-neighborly. Define the \emph{coherence parameter} for $(\X,\H)$ as
\begin{displaymath}
c^*(\X,\H)\triangleq\max_{a\in\A}\min_\mu\max_{\vec{h}\in\H}\sum_{i=1}^Nh(\x_i,a)\mu(\X_i),
\end{displaymath}
where $\mu$ is a probability measure over $\X$. Then, for any probability distribution $p$ over $\H$, one of the two statements below holds:
\begin{enumerate}
\item There is a set $\X_i\in\Xi$ such that
\begin{displaymath}
W(p,\x_i)\leq c^*.
\end{displaymath}
\item There are two $k$-neighbor sets $\X_i$ and $\X_j$ such that
\begin{align*}
W(p,\x_i)&>c^* &
W(p,\x_j)&>c^* \\
A^*(p,\x_i)&\neq 
A^*(p,\x_j).
\end{align*}
\end{enumerate}
\end{lemma}
\begin{proof}
See Appendix~\ref{Proof:Incoherence}.
\end{proof}

This lemma states that, given any distribution over the set of hypothesis, either there is a state $\x_i$ for which there is great uncertainty concerning the optimal action or, alternatively, there are two $k$-neighboring states $\x_i$ and $\x_j$ in which all except a few hypothesis predict the same action, yet the predicted optimal action is strikingly different in both states. In either case, it is possible to select a query that is highly informative. 

The coherence parameter $c^*$ is the multi-class equivalent of the coherence parameter introduced by \citet{nowak11tit}, and quantifies the informativeness of queries. That $c^*$ always exists can be established by noting that the partition of $\X$ is finite (since $\H$ is finite) and, therefore, the minimization can be conducted exactly. On the other hand, if $\H$ does not include trivial hypotheses that are constant all over $\X$, it holds that $c^*<1$.

We are now in position to establish the convergence properties of Algorithm~\ref{Alg:ActIRL}. Let $\PP{\cdot}$ and $\EE{\cdot}$ denote the probability measure and corresponding expectation governing the underlying probability over noise and possible algorithm randomizations in query selection. 

\begin{theorem}[Consistency of GBS-IRL]\label{Theo:Consistency}
Let $\F_t=\set{(x_\tau,a_\tau),\tau=1,\ldots,t}$ denote a possible history of observations obtained with GBS-IRL. If, in the update \eqref{Eq:Bayes}, $\hbeta(x)$ and $\hgamma(x)$ verify \eqref{Eq:Steps}, then 
\begin{displaymath}
\lim_{t\to\infty}\mathbb{P}\big[\hat{\vec{h}}_t\neq h^*\big]=0.
\end{displaymath}
\end{theorem}
\begin{proof}
See Appendix~\ref{Proof:Consistency}.
\end{proof}

Theorem~\ref{Theo:Consistency} establishes the consistency of active learning for multi-class classification. The proof relies on a fundamental lemma that, roughly speaking, ensures that the sequence $p_t(\vec{h}^*)$ is increasing in expectation. This fundamental lemma (Lemma~\ref{Lemma:Supermartingale} in Appendix~\ref{Proof:Consistency}) generalizes a related result of \cite{nowak11tit} that, due to the consideration of multiple classes in GBS-IRL, does not apply. Our generalization requires, in particular, stronger assumptions on the noise, $\hbeta(x)$, and implies a different rate of convergence, as will soon become apparent. It is also worth mentioning that the statement in Theorem~\ref{Theo:Consistency} could alternatively be proved using an adaptive sub-modularity argument (again relying on Lemma~\ref{Lemma:Supermartingale} in Appendix~\ref{Proof:Consistency}), using the results of \cite{golovin11jair}.

Theorem~\ref{Theo:Consistency} ensures that, as the number of samples increases, the probability mass concentrates on the correct hypothesis $\vec{h}^*$. However, it does not provide any information concerning the rate at which $\mathbb{P}\big[\hat{\vec{h}}_t\neq h^*\big]\to0$. The convergence rate for our active sampling approach is established in the following result.

\begin{theorem}[Convergence Rate of GBS-IRL]\label{Theo:Convergence}
Let $\H$ denote our hypothesis space, defined over $\X$, and assume that $(\X,\H)$ is 1-neighborly. If, in the update \eqref{Eq:Bayes}, $\hbeta(x)>\alpha$ for all $x\in\X$, then
\begin{equation}\label{Eq:ExponentialRate}
\mathbb{P}\big[\hat{\vec{h}}_t\neq\vec{h}^*\big]\leq\abs{\H}(1-\lambda)^t,\quad t=0,\ldots
\end{equation}
where $\lambda=\eps\cdot\min\set{\frac{1-c^*}{2},\frac{1}{4}}<1$ and
\begin{equation}\label{Eq:eps}
\eps=\min_x\gamma^*(x)\frac{\hgamma(x)-\hbeta(x)}{\hgamma(x)}+\beta^*(x)\frac{\hbeta(x)-\hgamma(x)}{\hbeta(x)}.
\end{equation}
\end{theorem}
\begin{proof}
See Appendix~\ref{Proof:Convergence}.
\end{proof}
Theorem~\ref{Theo:Convergence} extends Theorem~4 of \cite{nowak11tit} to the multi-class case. However, due to the existence of multiple actions (classes), the constants obtained in the above bounds differ from those obtained in the aforementioned work \citep{nowak11tit}. Interestingly, for $c^*$ close to zero, the convergence rate obtained is near-optimal, exhibiting a logarithmic dependence on the dimension of the hypothesis space. In fact, we have the following straightforward corollary of Theorem~\ref{Theo:Convergence}.

\begin{corollary}[Sample Complexity of GBS-IRL]\label{Cor:SampleComplexity}
Under the conditions of Theorem~\ref{Theo:Convergence}, for any given $\delta>0$, $\mathbb{P}\big[\hat{\vec{h}}_t=\vec{h}^*\big]>1-\delta$ as long as 
\begin{displaymath}
t\geq\frac{1}{\lambda}\log\frac{\abs{\H}}{\delta}.
\end{displaymath}
\end{corollary}

To conclude this section, we note that our reduction of IRL to a standard (multi-class) classification problem implies that Algorithm~\ref{Alg:ActIRL} is not specialized in any particular way to IRL problems---in particular, it can be used in general classification problems. Additionally, the guarantees in Theorems~\ref{Theo:Consistency} and \ref{Theo:Convergence} are also generally applicable in any multi-class classification problems verifying the corresponding assumptions.


\subsection{Discussion and Extensions}%
\label{Subsec:MultiAction}

We now discuss the general applicability of our results from Section~\ref{Subsec:SingleAction}. In particular, we discuss two assumptions considered in Theorem~\ref{Theo:Convergence}, namely the $1$-neighborly condition on $(\X,\H)$ and Assumption~\ref{Ass:OneAction}. We also discuss how additional forms of expert feedback may be integrated in a seamless manner in our GBS-IRL approach, so that the learner is able to recover efficiently the target task.


\subsubsection*{$1$-Neighborly Assumption:} 

This assumption is formulated in Theorem~\ref{Theo:Convergence}. The $1$-neighborly assumption states that $(\X,\H)$ is $1$-neighborly, meaning that it is possible to ``structure'' the state-space $\X$ in a manner that is coherent with the hypothesis space $\H$. To assess the validity of this assumption in general, we start by recalling that two sets $\X_i,\X_j\in\Xi$ are 1-neighbors if there is a single hypothesis $\vec{h}_0\in\H$ that prescribes different optimal actions in $\X_i$ and $\X_j$. Then, $(\X,\H)$ is 1-neighborly if every two sets $\X_i,\X_j$ can be ``connected'' by a sequence of 1-neighbor sets.

In general, given a multi-class classification problem with hypothesis space $\H$, the $1$-neighborly assumption can be investigated by verifying the connectivity of the $1$-neighborhood graph induced by $\H$ on $\X$. We refer to the work of \cite{nowak11tit} for a detailed discussion of this case, as similar arguments carry to our multi-class extension.

In the particular case of inverse reinforcement learning, it is important to assess whether the $1$-neighborly assumption is reasonable. Given a finite state-space, $\X$, and a finite action-space, $\A$, it is possible to build a total of $\abs{\A}^{\abs{\X}}$ different hypothesis.%
\footnote{This number is even larger if multiple optimal actions are allowed.}
As shown in the work of \cite{melo10ecai}, for any such hypothesis it is always possible to build a non-degenerate reward function that yields such hypothesis as the optimal policy. Therefore, a sufficiently rich reward space ensures that the corresponding hypothesis space $\H$ includes all $\abs{\A}^{\abs{\X}}$ possible policies already alluded to. This trivially implies that $(\X,\H)$ is \emph{not} $1$-neighborly. 

Unfortunately, as also shown in the aforementioned work \citep{melo10ecai}, the consideration of $\H$ as the set of all possible policies also implies that all states must be sufficiently sampled, since no generalization across states is possible. This observation supports the option in most IRL research to focus on problems in which rewards/policies are selected from some restricted set \citep[for example,][]{abbeel04icml,ramachandran07ijcai,neu07uai,syed08nips}. For the particular case of active learning approaches, the consideration of a full set of rewards/policies also implies that there is little hope that any active sampling will provide any but a negligible improvement in sample complexity. A related observation can be found in the work of \cite{dasgupta04nips} in the context of active learning for binary classification.

\begin{algorithm}[!tb]
\footnotesize
\caption{GBS-IRL (version~2)}\label{Alg:SimpleIRL}
  \begin{algorithmic}[1]
    \REQUIRE MDP parameters $\M\backslash{r}$
    \REQUIRE Reward space $\R_\rho$
    \REQUIRE Prior distribution $p_0$ over $\R$
    \STATE Compute $\H$ from $\R_\rho$
    \STATE Determine partition $\Xi={\X_1,\ldots\X_N}$ of $\X$
    \STATE Set $\F_0=\emptyset$
    \FORALL{$t=0,\ldots$}
      \STATE Sample $x_{t+1}$ from the set $\X_i$ that minimizes $W(p_t,\x_i)$.
      \STATE Obtain noisy response $a_{t+1}$
	    \STATE Set $\F_{t+1}\leftarrow\F_t\cup\set{(x_{t+1},a_{t+1})}$
      \STATE Update $p_{t+1}$ from $p_t$ using \eqref{Eq:Bayes}
    \ENDFOR
    \RETURN $\hat{\vec{h}}_t=\argmax_{\vec{h}\in\H}p_t(\vec{h})$.
  \end{algorithmic}
\end{algorithm}

In situations where the $1$-neighborly assumption may not be verified, Lemma~\ref{Lemma:Incoherence} cannot be used to ensure the selection of highly informative queries once $W(p,\x_i)>c^*$ for all $\X_i$. However, it should still be possible to use the main approach in GBS-IRL, as detailed in Algorithm~\ref{Alg:SimpleIRL}. For this situation, we can specialize our sample complexity results in the following immediate corollary.

\begin{corollary}[Convergence Rate of GBS-IRL, version~2]\label{Cor:Convergence}
Let $\H$ denote our hypothesis space, defined over $\X$, and let $\hbeta(x)>\alpha$ in the update \eqref{Eq:Bayes}. Then, for all $t$ such that $W(p_t,\x_i)\leq c^*$ for some $\X_i$, 
\begin{displaymath}
\mathbb{P}\big[\hat{\vec{h}}_t\neq\vec{h}^*\big]\leq\abs{\H}(1-\lambda)^t,\quad t=0,\ldots
\end{displaymath}
where $\lambda=\eps\frac{1-c^*}{2}$ and $\eps$ is defined in \eqref{Eq:eps}.
\end{corollary}


\subsubsection*{Multiple Optimal Actions:} 

In our presentation so far, we assumed that $\R_\rho$ is such that, for any $r\in\R_\rho$ and any $x\in\X$, $\abs{\A_r(x)}=1$ (Assumption~\ref{Ass:OneAction}). Informally, this corresponds to assuming that, for every reward function considered, there is a single optimal action, $\pi^*(x)$, at each $x\in\X$. This assumption has been considered, either explicitly or implicitly, in several previous works on learning by demonstration \citep[see, for example, the work of][]{chernova09jair}. Closer to our own work on active IRL, several works recast IRL as a classification problem, focusing on deterministic policies $\pi_k:\X\to\A$ \citep{ng00icml,cohn11aaai,judah11icml,ross10aistats,ross11aistats} and therefore, although not explicitly, also consider a single optimal action in each state.

However, MDPs with multiple optimal actions per state are not uncommon (the scenarios considered in Section~\ref{Sec:Experiments}, for example, have multiple optimal actions per state). In this situation, the properties of the resulting algorithm do not follow from our previous analysis, since the existence of multiple optimal actions necessarily requires a more general noise model. The immediate extension of our noise model to a scenario where multiple optimal actions are allowed poses several difficulties, as optimal actions across policies may be sampled with different probabilities. 

In order to overcome such difficulty, we consider a more conservative Bayesian update, that enables a seamless generalization of our results to scenarios that admit multiple optimal actions in each state. Our update now arises from considering that the likelihood of observing an action from a set $\A_{\vec{h}}(x)$ at state $x$ is given by $\gammah(x)$. Equivalently, the likelihood of observing an action from $\A-\A_{\vec{h}}(x)$ is given by $\betah(x)=1-\gammah(x)$. As before, $\gamma^*$ and $\beta^*$ correspond to the values of $\gammah$ and $\betah$ for the target hypothesis, and we again let 
\begin{displaymath}
\alpha=\sup_{x\in\X}\beta^*(x).
\end{displaymath}
Such \emph{aggregated} noise model again enables the consideration of an approximate noise model that is constant across hypothesis, and is defined in terms of estimates $\hgamma(x)$ and $\hbeta(x)$ of $\gamma^*(x)$ and $\beta^*(x)$. Given the noise model just described, we get the Bayesian update
\begin{align}
\nonumber%
p_t(\vec{h})
\nonumber%
   &\triangleq\PP{\vec{h}^*=\vec{h}\mid\F_t}\\
\nonumber%
   &\propto\PP{a_t\in\A_{\vec{h}}\mid x_t,\F_{t-1}}\PP{\vec{h}^*=\vec{h}\mid\F_{t-1}}\\
\nonumber%
   &=\PP{a_t\in\A_{\vec{h}}\mid x_t}\PP{\vec{h}=\vec{h}^*\mid\F_{t-1}}\\
\label{Eq:Bayes-new}
   &\approx\hgamma(x)^{(1+h_t)/2}\hbeta(x)^{(1-h_t)/2}p_{t-1}(\vec{h}),
\end{align}
with $\hgamma(x)$ and $\hbeta(x)$ verifying \eqref{Eq:Steps}. This revised formulation implies that the updates to $p_t$ are more conservative, in the sense that they are slower to ``eliminate'' hypothesis from $\H$. However, all results for Algorithm~\ref{Alg:ActIRL} remain valid with the new values for $\hgamma$ and $\hbeta$.

\begin{algorithm}[!tb]
\footnotesize
\caption{GBS-IRL (version~3)}\label{Alg:SimpleIRL-2}
  \begin{algorithmic}[1]
    \REQUIRE MDP parameters $\M\backslash{r}$
    \REQUIRE Reward space $\R_\rho$
    \REQUIRE Prior distribution $p_0$ over $\R$
    \STATE Compute $\H$ from $\R_\rho$
    \STATE Determine partition $\Xi={\X_1,\ldots\X_N}$ of $\X$
    \STATE Set $\F_0=\emptyset$
    \FORALL{$t=0,\ldots$}
    \STATE Set $c_t=\min_{i=1,\ldots,N}W(p_t,\x_i)$
    \IF{$c_t<\hat{c}$}
    \STATE Sample $x_{t+1}$ from the set $\X_i$ that minimizes $W(p_t,\x_i)$.
    \ELSE
    \STATE Return $\hat{\vec{h}}_t=\argmax_{\vec{h}\in\H}p_t(\vec{h})$.
    \ENDIF
    \STATE Obtain noisy response $a_{t+1}$
    \STATE Set $\F_{t+1}\leftarrow\F_t\cup\set{(x_{t+1},a_{t+1})}$
    \STATE Update $p_{t+1}$ from $p_t$ using \eqref{Eq:Bayes-new}
    \ENDFOR
  \end{algorithmic}
\end{algorithm}

Unfortunately, by allowing multiple optimal actions per state, it is also much easier to find (non-degenerate) situations where $c^*=1$, in which case our bounds are void. However, if we focus on identifying, in each state, at least \emph{one optimal action}, we are able to retrieve some guarantees on the sample complexity of our active learning approach. We thus consider yet another version of GBS-IRL, described in Algorithm~\ref{Alg:SimpleIRL-2}, that uses uses a threshold $\hat{c}<1$ such that, if $W(p_t,\x_i)>\hat{c}$, we consider that (at least) one optimal action at $\x_i$ has been identified. Once this is done, it outputs the most likely hypothesis. Once at least one optimal action has been identified in all states, the algorithm stops. 

To analyze the performance of this version of GBS-IRL, let the set of \emph{predicted optimal actions at $x$} be defined as
\begin{displaymath}
\A_{\hat{c}}(p,x)=\set{a\in\A\mid\sum_{\vec{h}}p(\vec{h})h(x,a)>\hat{c}}.
\end{displaymath}
We have the following results.

\begin{theorem}[Consistency of GBS-IRL, version~3]\label{Theo:Consistency-general}
Consider any history of observations $\F_t=\set{(x_\tau,a_\tau),\tau=1,\ldots,t}$ from GBS-IRL. If, in the update \eqref{Eq:Bayes}, $\hbeta$ and $\hgamma$ verify \eqref{Eq:Steps} for all $\vec{h}\in\H$, then for any $a\in\A_{\hat{c}}(p, \x_i)$,
\begin{displaymath}
\lim_{t\to\infty}\mathbb{P}\big[h^*(\x_i,a)\neq1\big]=0.
\end{displaymath}
\end{theorem}
\begin{proof}
See Appendix~\ref{Proof:Consistency-general}.
\end{proof}

Note that the above result is no longer formulated in terms of the identification of the correct hypothesis, but in terms of the identification of the set of optimal actions. We also have the following result on the sample complexity of version 3 of GBS-IRL.

\begin{corollary}[Convergence Rate of GBS-IRL, version~3]\label{Cor:Convergence-2}
Let $\H$ denote our hypothesis space, defined over $\X$, and let $\hbeta(x)>\alpha$ in the update \eqref{Eq:Bayes-new}. Then, for all $t$ such that $W(p_t,\x_i)\leq c^*$ for some $\X_i$, and all $a\in\A_{\hat{c}}(p, \x_i)$,
\begin{displaymath}
\mathbb{P}\big[h^*(\x_i,a)\neq1\big]\leq\abs{\H}(1-\lambda)^t,\quad t=0,\ldots
\end{displaymath}
where $\lambda=\eps\frac{1-c^*}{2}$ and $\eps$ is defined in \eqref{Eq:eps} with the new values for $\hgamma$ and $\hbeta$.
\end{corollary}


\subsubsection*{Different Query Types:} 

Finally, it is worth noting that, in the presentation so far admits for queries such as \textit{``What is the optimal action in state $x$?''} However, it is possible to devise different types of queries (such as \textit{``Is action $a$ optimal in state $x$?''}) that enable us to recover the stronger results in Theorem~\ref{Theo:Convergence}. In fact, a query such as the one exemplified reduces the IRL problem to a \emph{binary classification problem} over $\X\times\A$, for which existing active learning methods such as the one of \cite{nowak11tit} can readily be applied.


\subsubsection*{Integrating Reward Feedback:} 

So far, we discussed one possible approach to IRL, where the agent is provided with a demonstration $\F_t=\set{(x_\tau,a_\tau),\tau=1,\ldots,t}$ consisting of pairs $(x_\tau,a_\tau)$ of states and corresponding actions. From this demonstration the agent must identify the underlying target task, represented as a reward function, $r^*$. We now depart from the Bayesian formalism introduced above and describe how reward information can also be integrated. 

With the addition of reward information, our demonstrations may now include state-reward pairs $(x_\tau, u_\tau)$, indicating that the reward in state $x_\tau$ takes the value $u_\tau$. This can be seen as a similar approach as those of \cite{thomaz08aij,knox10aamas} for reinforcement learning. The main difference is that, in the aforementioned works, actions are experienced by the learner who then receives rewards both from the environment and the teacher.  Another related approach is introduced by \cite{regan11ijcai}, in the context of reward design for MDPs.

As with action information, the demonstrator would ideally provide exact values for $r^*$. However, we generally allow the demonstration to include some level of noise, where
\begin{equation}\label{Eq:RewardDist}
\PP{u_\tau=u\mid x_\tau,r^*}\propto e^{(u-\rt(x))^2/\sigma},
\end{equation}
where $\sigma$ is a non-negative constant. As with policy information, reward information can be used to update $p_t(r_k)$ as
\begin{align*}
p_t(r_k)
   &\triangleq\PP{r^*=r_k\mid\F_t}\\
   &\propto\PP{u_t\mid x_t,r^*=r_k,\F_{t-1}}\PP{r^*=r_k\mid\F_{t-1}}\\
   &=\PP{u_t\mid x_t,r^*=r_k}\PP{r^*=r_k\mid\F_{t-1}}\\
   &\approx e^{(u_t-r_k(x))^2/\hsigma}p_{t-1}(r_k)
\end{align*}
where, as before, we allow for an inaccurate estimate $\hsigma$ of $\sigma$ such that $\hsigma\geq\sigma$. Given the correspondence between the rewards in $\R_\rho$ and the hypothesis in $\H$, the above Bayesian update can be used to seamlessly integrate reward information in our Bayesian IRL setting.

To adapt our active learning approach to accommodate for reward feedback, let 
\begin{displaymath}
x_{t+1}=\argmin_{\X_i,i=1,\ldots,N}W(p_t,\x_i).
\end{displaymath}
\ie $x_{t+1}$ is the state that would be queried by Algorithm~\ref{Alg:ActIRL} at time-step $t+1$. If the user instead wishes to provide reward information, we would like to replace the query $x_{t+1}$ by some alternative query $x'_{t+1}$ that disambiguates as much as possible the actions in state $x_{t+1}$---much like a direct query to $x_{t+1}$ would. 

To this purpose, we partition the space of rewards, $\R_\rho$, into $\abs{\A}$ or less disjoint sets $\R_1,\ldots,\R_{\abs{\A}}$, where each set $\R_a$ contains precisely those rewards $r\in\R_\rho$ for which $\pi_r(x_{t+1})=a$. We then select the state $x'_{t+1}\in\X$, the reward at which best discriminates between the sets $\R_1,\ldots,\R_{\abs{\A}}$. The algorithm will then query the demonstrator for the reward at this new state.

In many situations, the rewards in $\R_\rho$ allow only poor discrimination between the sets $\R_1,\ldots,\R_{\abs{\A}}$. This is particularly evident if the reward is sparse, since after a couple informative reward samples, all other states contain similar reward information. In Section~\ref{Sec:Experiments} we illustrate this inconvenience, comparing the performance of our active method in the presence of both sparse and dense reward functions.


\section{Experimental Results}%
\label{Sec:Experiments}

This section illustrates the application of GBS-IRL in several problems of different complexity. It also features a comparison with other existing methods from the active IRL literature.


\subsection{GBS-IRL}
\label{Subsec:Active}

In order to illustrate the applicability of our proposed approach, we conducted a series of experiments where GBS-IRL is used to determine the (unknown) reward function for some underlying MDP, given a perturbed demonstration of the corresponding policy. 

In each experiment, we illustrate and discuss the performance of GBS-IRL. The results presented correspond to averages over 200 independent Monte-Carlo trials, where each trial consists of a run of 100 learning steps, in each of which the algorithm is required to select one state to query and is provided the corresponding action. GBS-IRL is initialized with a set $\R_\rho$ of 500 independently generated random rewards. This set always includes the correct reward, $r^*$ and the remaining rewards are built to have similar range and sparsity as that of $r^*$. 

The prior probabilities, $p_0(r)$, are proportional to the level of sparsity of each reward $r$. This implies that some of the random rewards in $\R_\rho$ may have larger prior probability than $r^*$. For simplicity, we considered an exact noise model, \ie $\hbeta=\beta^*$ and $\hgamma=\gamma^*$, where $\beta^*(x)\equiv0.1$ and $\gamma^*(x)\equiv0.9$, for all $x\in\X$. 

For comparison purposes, we also evaluated the performance of other active IRL approaches from the literature, to know:
\begin{itemize}
\item The \emph{imitation query-by-committee} algorithm (IQBC) of \cite{judah11icml}, that uses an entropy-based criterion to select the states to query. \item The \emph{expected myopic gain} algorithm (EMG) of \cite{cohn11aaai}, that uses a criterion based on the expected gain of information to select the states to query. 
\end{itemize}
As pointed out in Section~\ref{Subsec:Related}, IQBC is, in its core, very similar to GBS-IRL, the main differences being in terms of the selection criterion and of the fact that the IQBC is able to accommodate the notion of ``bad states''. Since this notion is not used in our examples, we expect the performance of both methods to be essentially similar. 

As for EMG, this algorithm queries the expert for the correct action in those states where the expected gain of information is potentially larger \citep{cohn11aaai}. This requires evaluating, for each state $x\in\X$ and each possible outcome, the associated gain of information. Such method is, therefore, fundamentally different from GBS-IRL and we expect this method to yield crisper differences from our own approach. Additionally, the above estimation is computationally heavy, as (in the worst case) requires the evaluation of an MDP policy for each state-action pair.


\subsubsection*{Small-sized random MDPs}

In the first set of experiments, we evaluate the performance of GBS-IRL in several small-sized MDPs with no particular structure (both in terms of transitions and in terms of rewards). Specifically, we considered MDPs where $\abs{\X}=10$ and either $\abs{\A}=5$ or $\abs{\A}=10$. For each MDP size, we consider 10 random and independently generated MDPs, in each of which we conducted 200 independent learning trials. This first set of experiments serves two purposes. On one hand, it illustrates the applicability of GBS-IRL in arbitrary settings, by evaluating the performance of our method in random MDPs with no particular structure. On the other hand, these initial experiments also enable a quick comparative analysis of GBS-IRL against other relevant methods from the active IRL literature.

\begin{figure}[!tb]
\centering
  \subfigure[]{\label{Fig:Policy-10-5}
    \includegraphics[width=0.45\columnwidth]{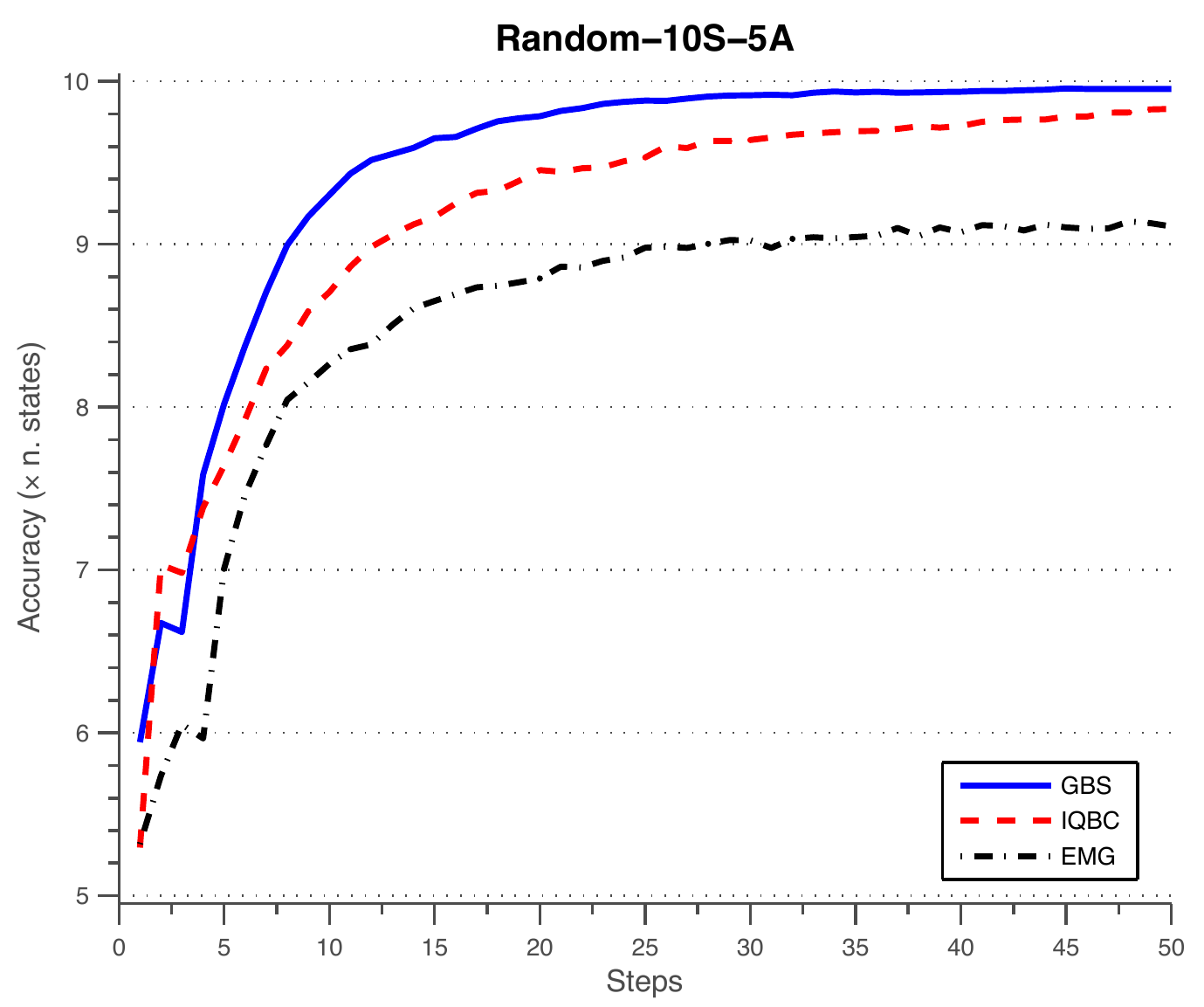}
  }\hfill
  \subfigure[]{\label{Fig:Value-10-5}
    \includegraphics[width=0.45\columnwidth]{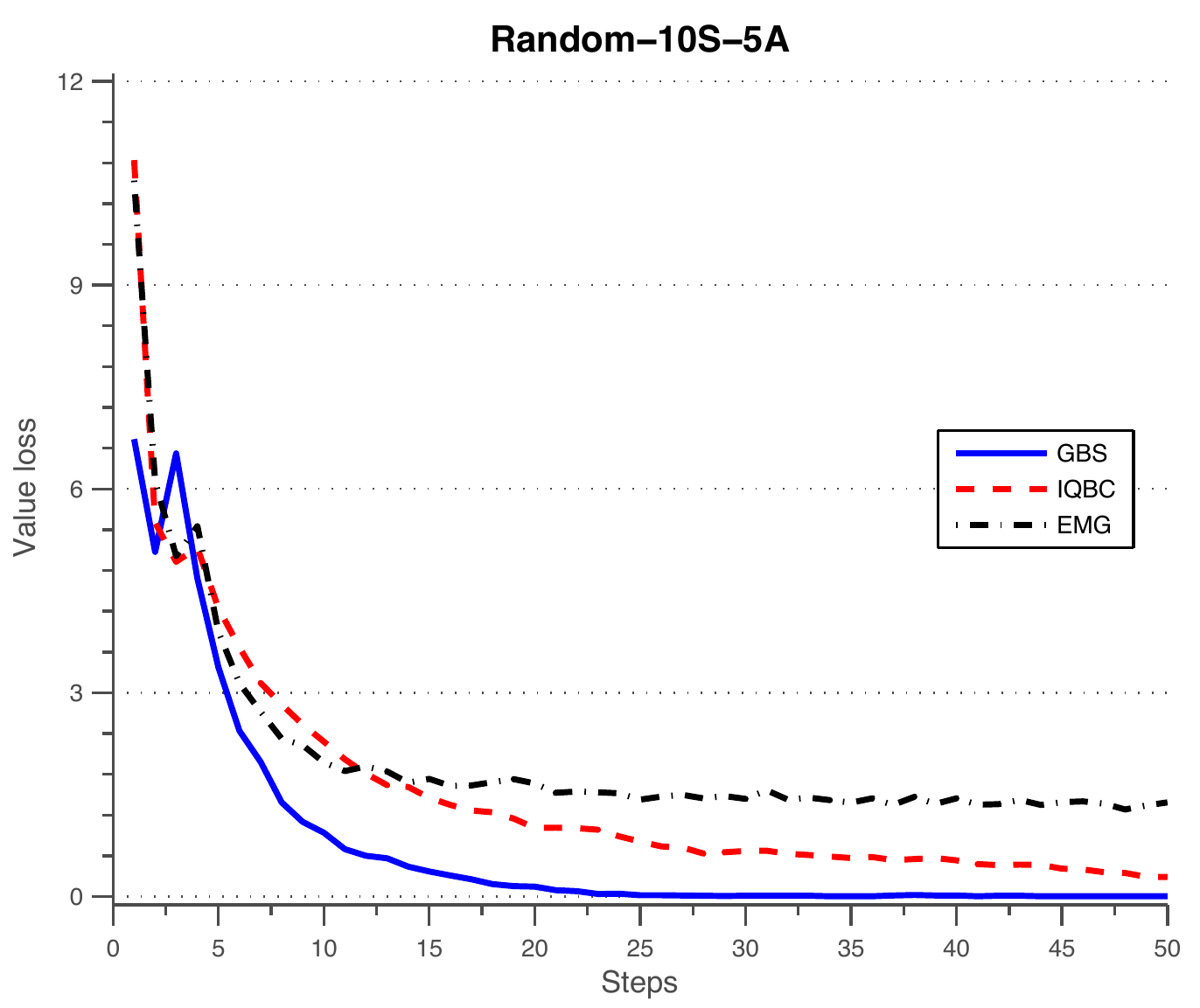}
  }
  \caption{Performance of all methods in random MDPs with $\abs{\X}=10$ and $\abs{\A}=5$.}
  \label{Fig:Results-10-5}
\end{figure}

\begin{figure}[!tb]
\centering
  \subfigure[]{\label{Fig:Policy-10-10}
    \includegraphics[width=0.45\columnwidth]{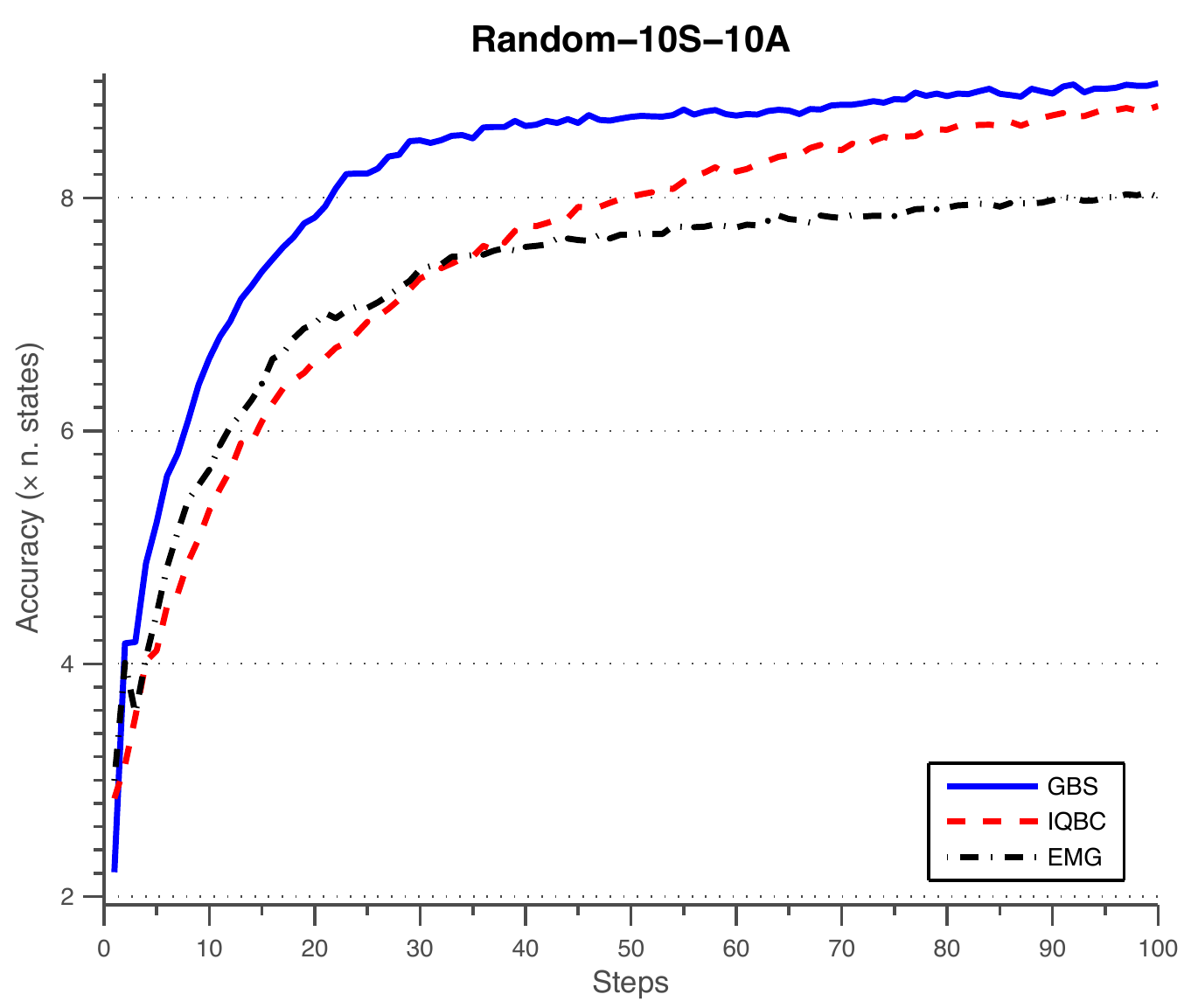}
  }\hfill
  \subfigure[]{\label{Fig:Value-10-10}
    \includegraphics[width=0.45\columnwidth]{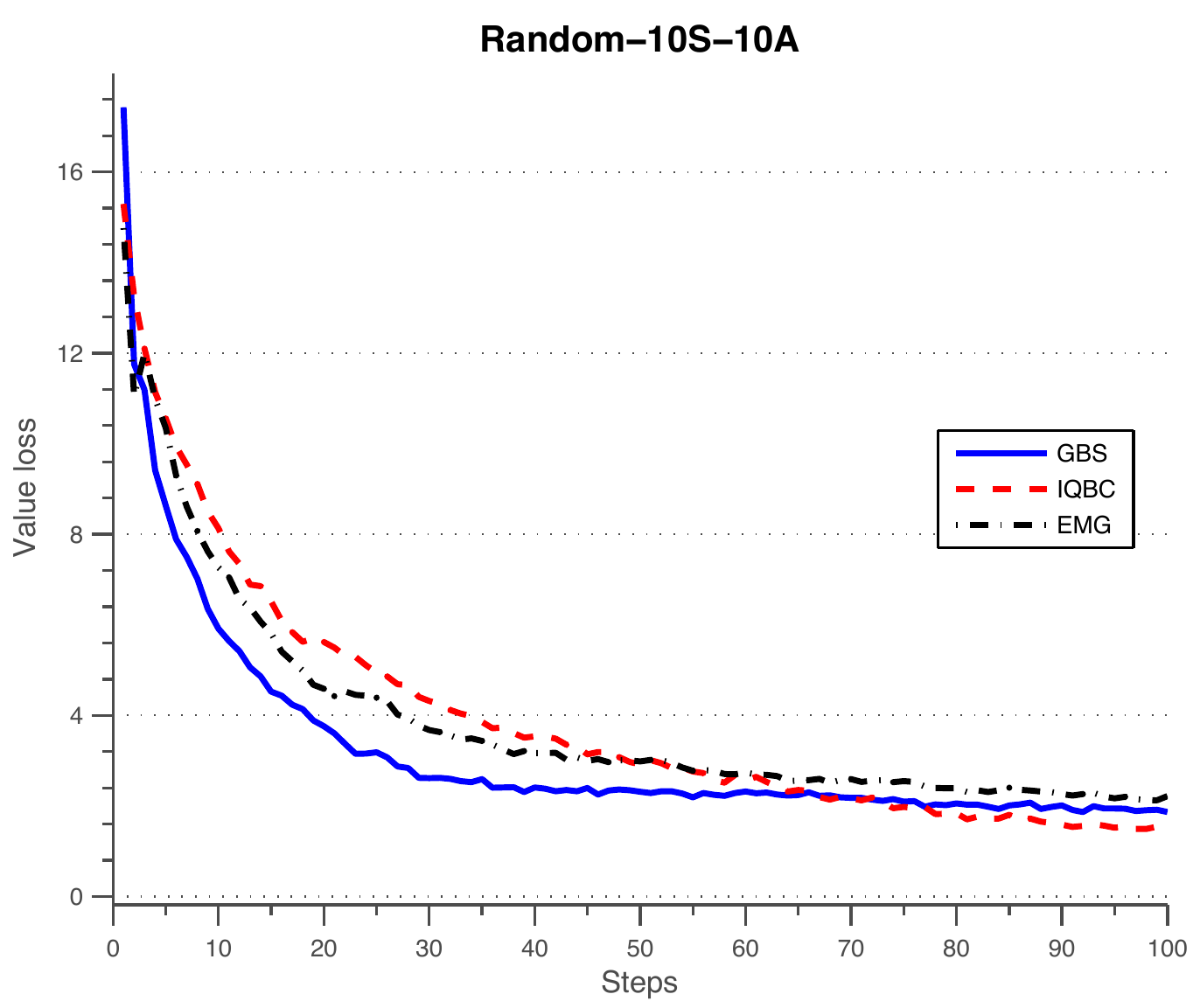}
  }
  \caption{Performance of all methods in random MDPs with $\abs{\X}=10$ and $\abs{\A}=10$.}
  \label{Fig:Value-10}
\end{figure}

Figures~\ref{Fig:Policy-10-5} and \ref{Fig:Policy-10-10} depict the learning curve for all three methods in terms of policy accuracy. The performance of all three methods is essentially similar in the early stages of the learning process. However, GBS-IRL slightly outperforms the other two methods, although the differences from IQBC are, as expected, smaller than those from EMG.

While policy accuracy gives a clear view of the learning performance of the algorithms, it conveys a less clear idea on the ability of the learned policies to complete the task intended by the demonstrator. To evaluate the performance of the three learning algorithms in terms of the target task, we also measured the loss of the learned policies with respect to the optimal policy. Results are depicted in Figs.~\ref{Fig:Value-10-5} and \ref{Fig:Value-10-10}. These results also confirm that the performance of GBS-IRL is essentially similar. In particular, the differences observed in terms of policy accuracy have little impact in terms of the ability to perform the target task competently. 

\begin{figure}[!tb]
\centering
  \includegraphics[width=0.6\columnwidth]{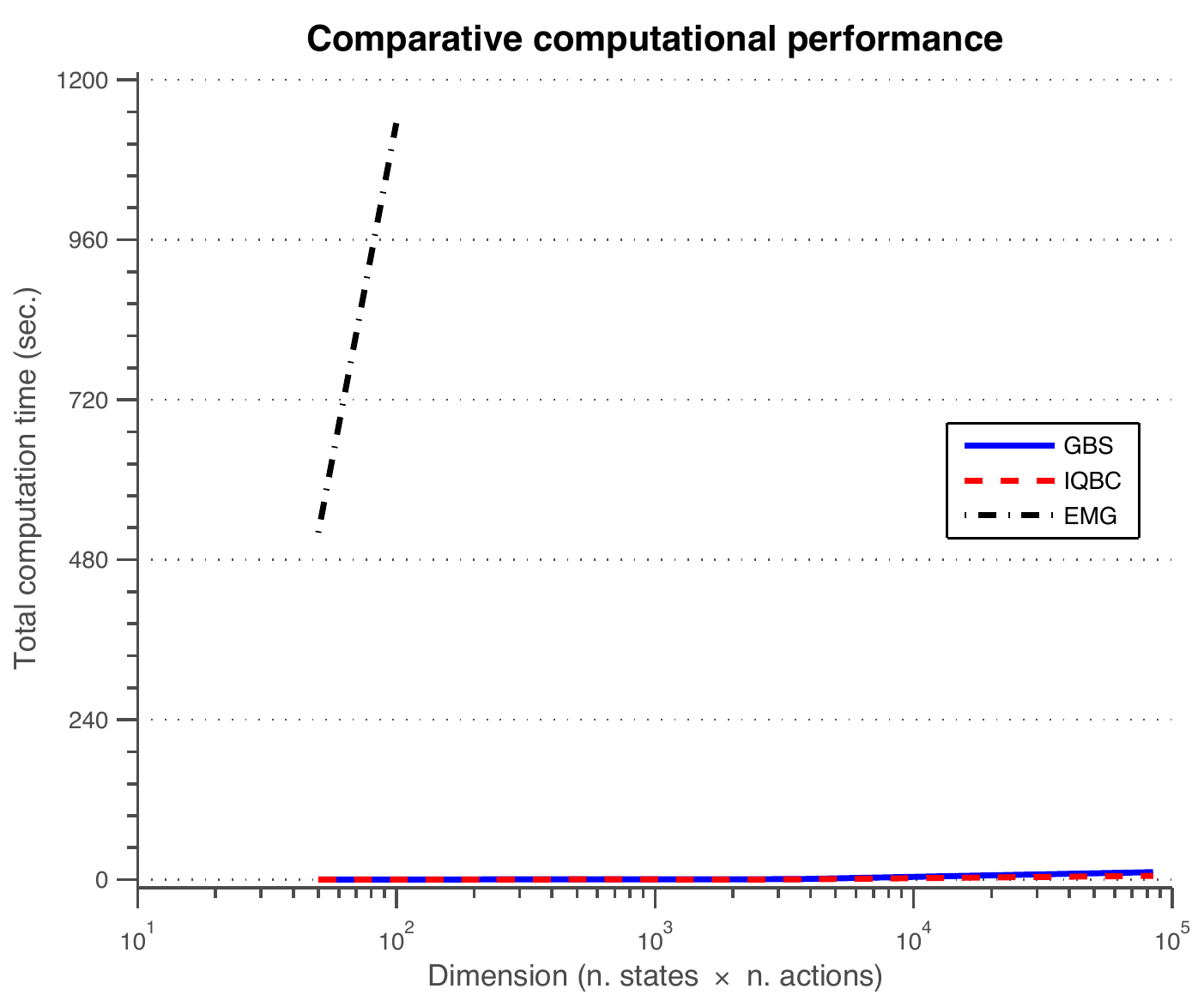}
  \caption{Average (total) computational time for problems of different dimensions.}
  \label{Fig:Time}
\end{figure}

To conclude this section, we also compare the computation time for all methods in these smaller problems. The results are depicted in Fig.~\ref{Fig:Time}. We emphasize that the results portrayed herein are only indicative, as all algorithms were implemented in a relatively straightforward manner, with no particular concerns for optimization. Still, the comparison does confirm that the computational complexity associated with EMG is many times superior to that involved in the remaining methods. This, discussed earlier, is due to the heavy computations involved in the estimation of the expected myopic gain, which grows directly with the size of $\abs{\X}\times\abs{\A}$. This observation is also in line with the discussion already found in the original work of \cite{cohn11aaai}.


\subsubsection*{Medium-sized random MDPs}

In the second set of experiments, we investigate how the performance of GBS-IRL is affected by the dimension of the domain considered. To this purpose, we evaluate the performance of GBS-IRL in arbitrary medium-sized MDPs with no particular structure (both in terms of transitions and in terms of rewards). Specifically, we now consider MDPs where either $\abs{\X}=50$ or $\abs{\X}=100$, and again take either $\abs{\A}=5$ or $\abs{\A}=10$. For each MDP size, we consider 10 random and independently generated MDPs, in each of which we conducted 200 independent learning trials. 

%

\begin{figure}[!tb]
\centering
  \subfigure[Policy accuracy ($50\times5$)]{\label{Fig:Policy-50-5}
    \includegraphics[width=0.4\columnwidth]{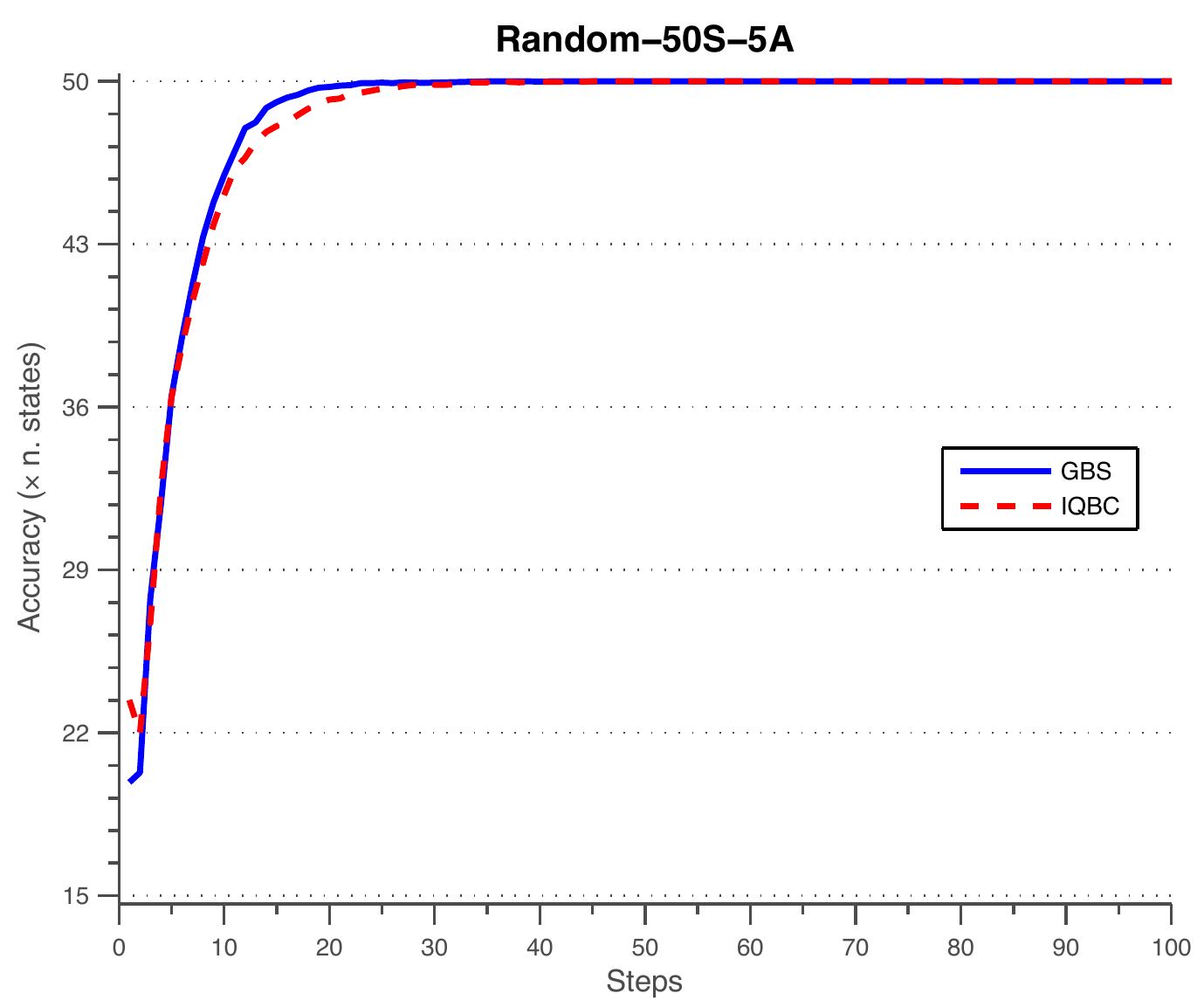}
  }\hfill
  \subfigure[Policy accuracy ($100\times5$)]{\label{Fig:Policy-100-5}
    \includegraphics[width=0.4\columnwidth]{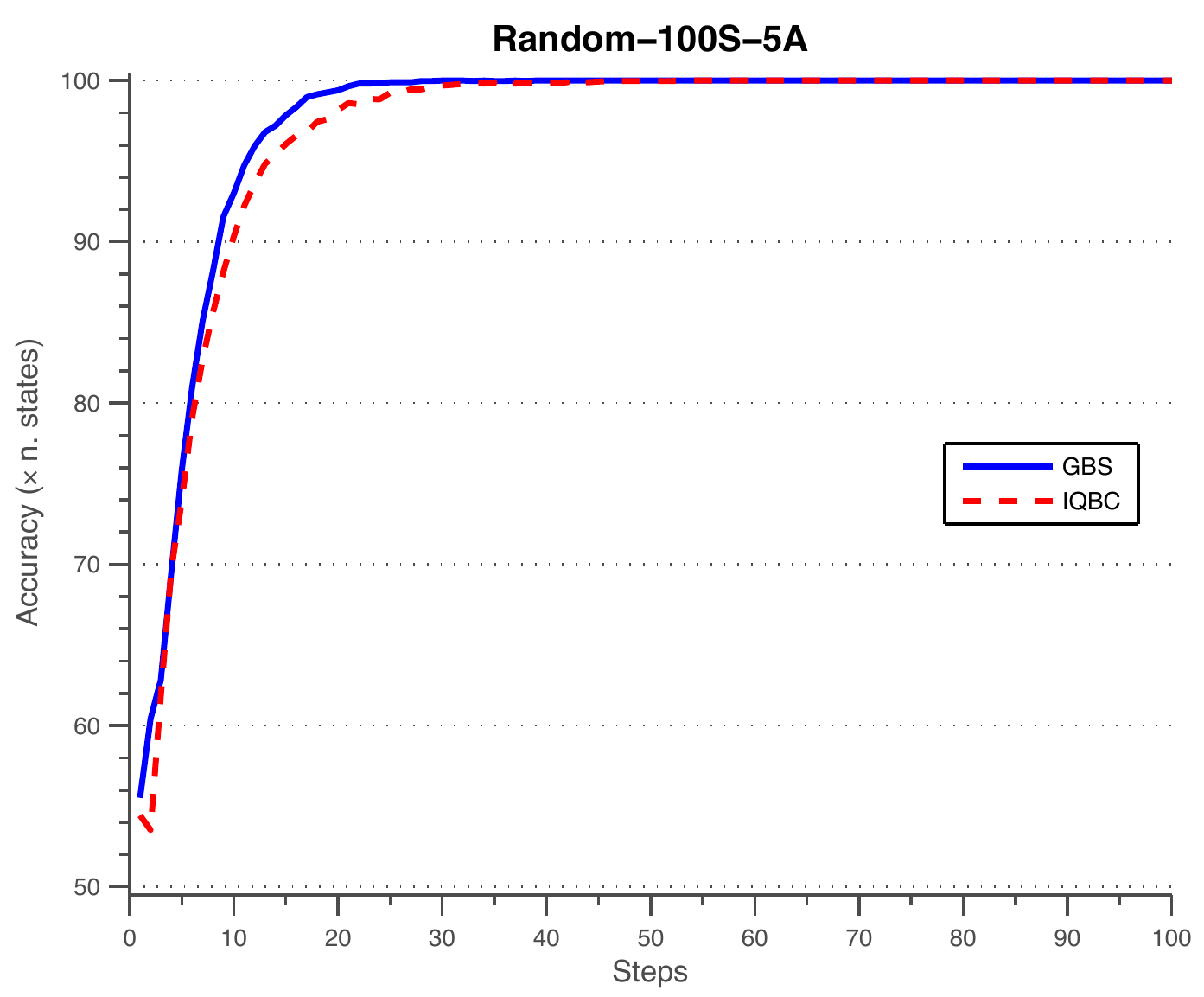}
  }\\
  \subfigure[Policy accuracy ($50\times10$)]{\label{Fig:Policy-50-10}
    \includegraphics[width=0.4\columnwidth]{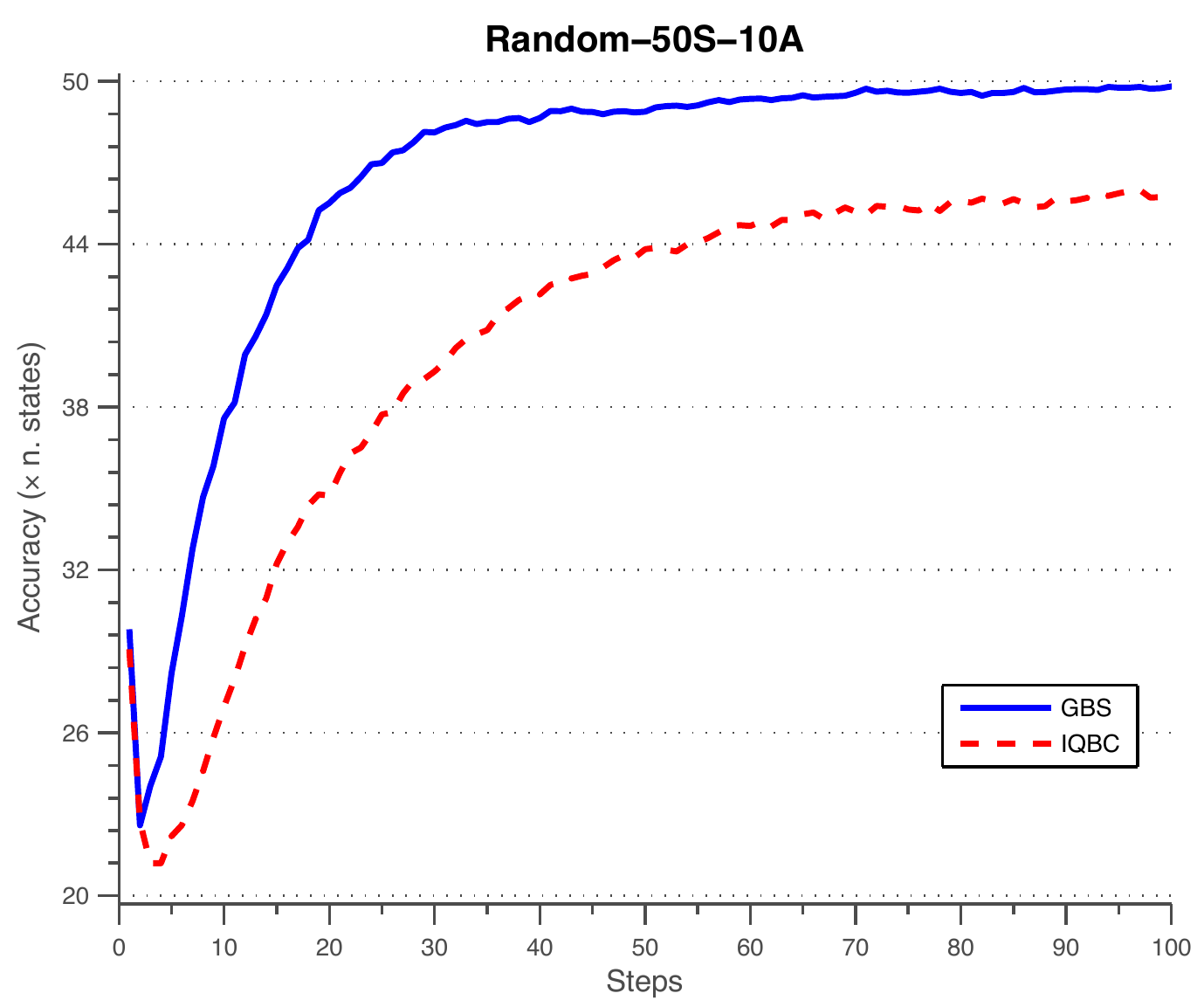}
  }\hfill
  \subfigure[Value perf.\ ($50\times5$)]{\label{Fig:Value-50-5}
    \includegraphics[width=0.4\columnwidth]{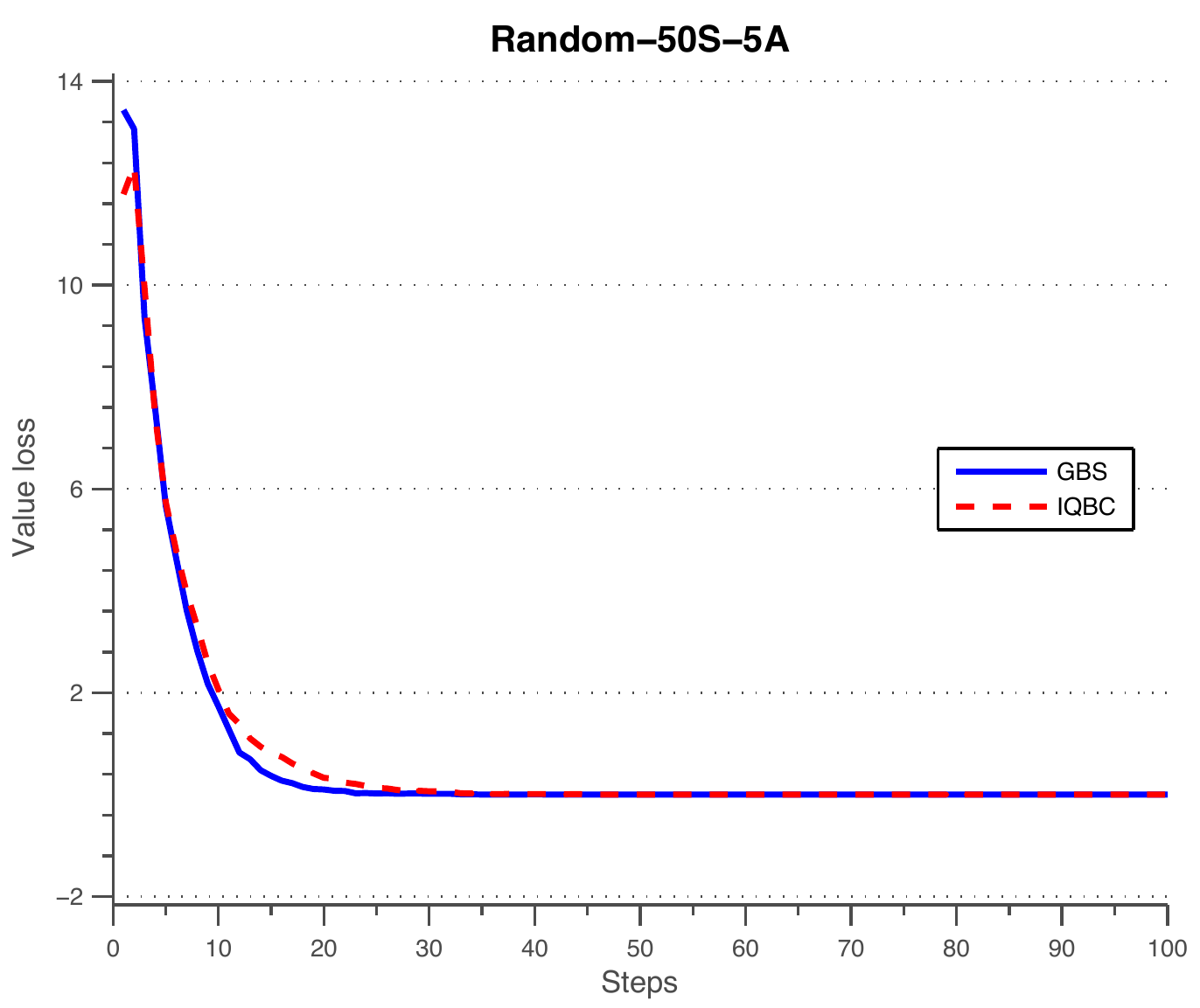}
  }\\
  \subfigure[Value perf.\ ($100\times5$)]{\label{Fig:Value-100-5}
    \includegraphics[width=0.4\columnwidth]{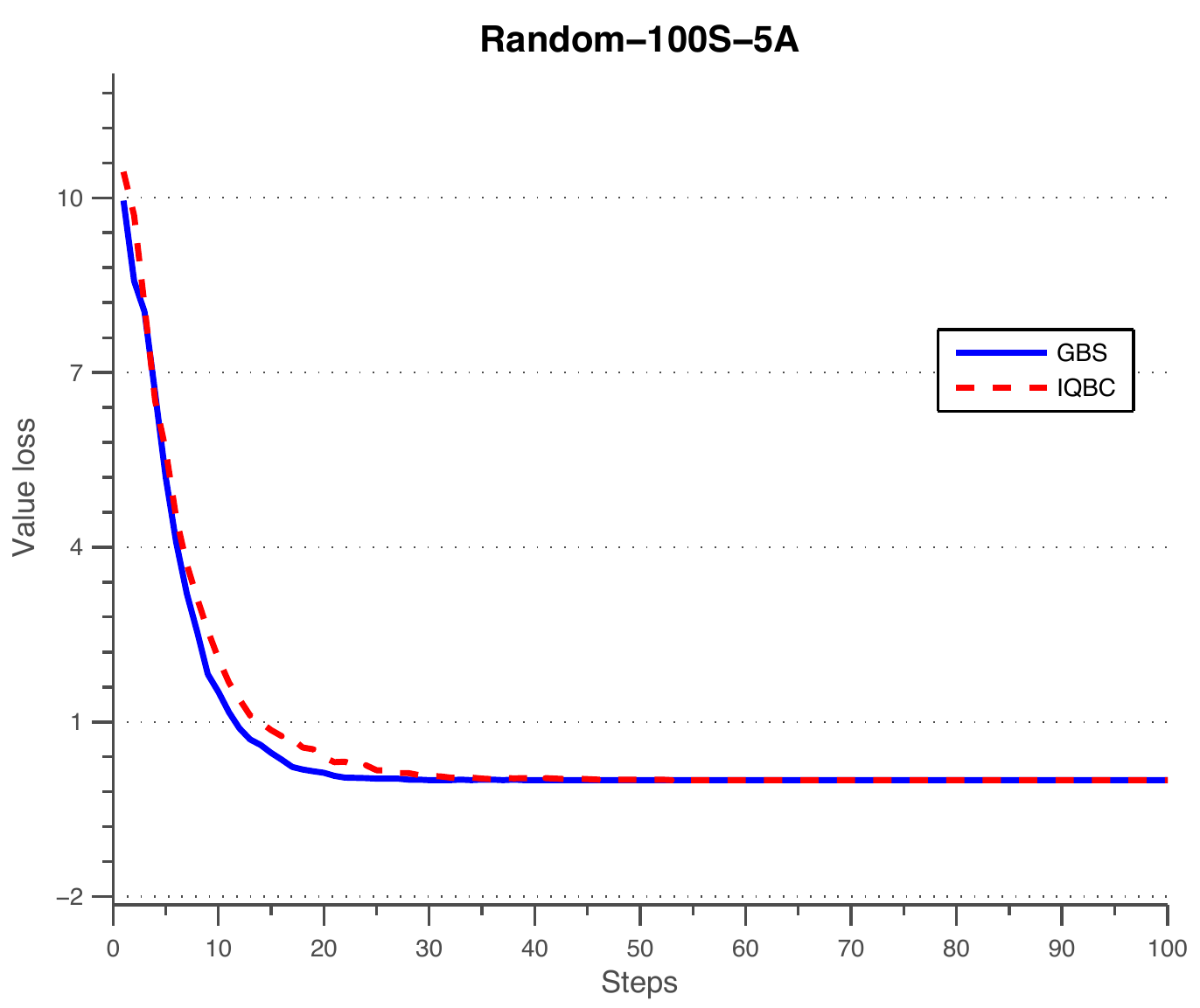}
  }\hfill
  \subfigure[Value perf.\ ($50\times10$)]{\label{Fig:Value-50-10}
    \includegraphics[width=0.4\columnwidth]{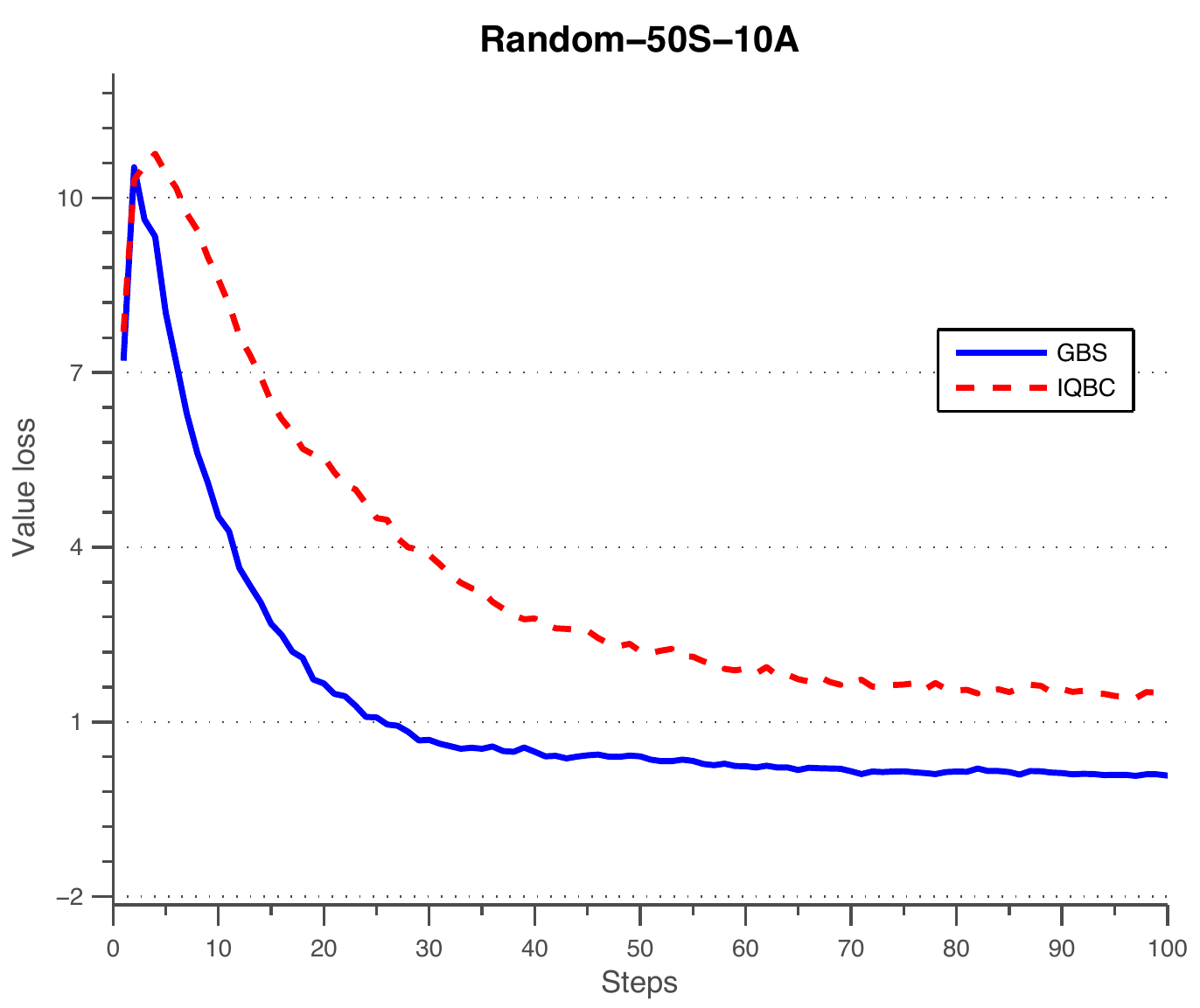}
  }\\
  \subfigure[Policy accuracy ($100\times10$)]{\label{Fig:Policy-100-10}
    \includegraphics[width=0.4\columnwidth]{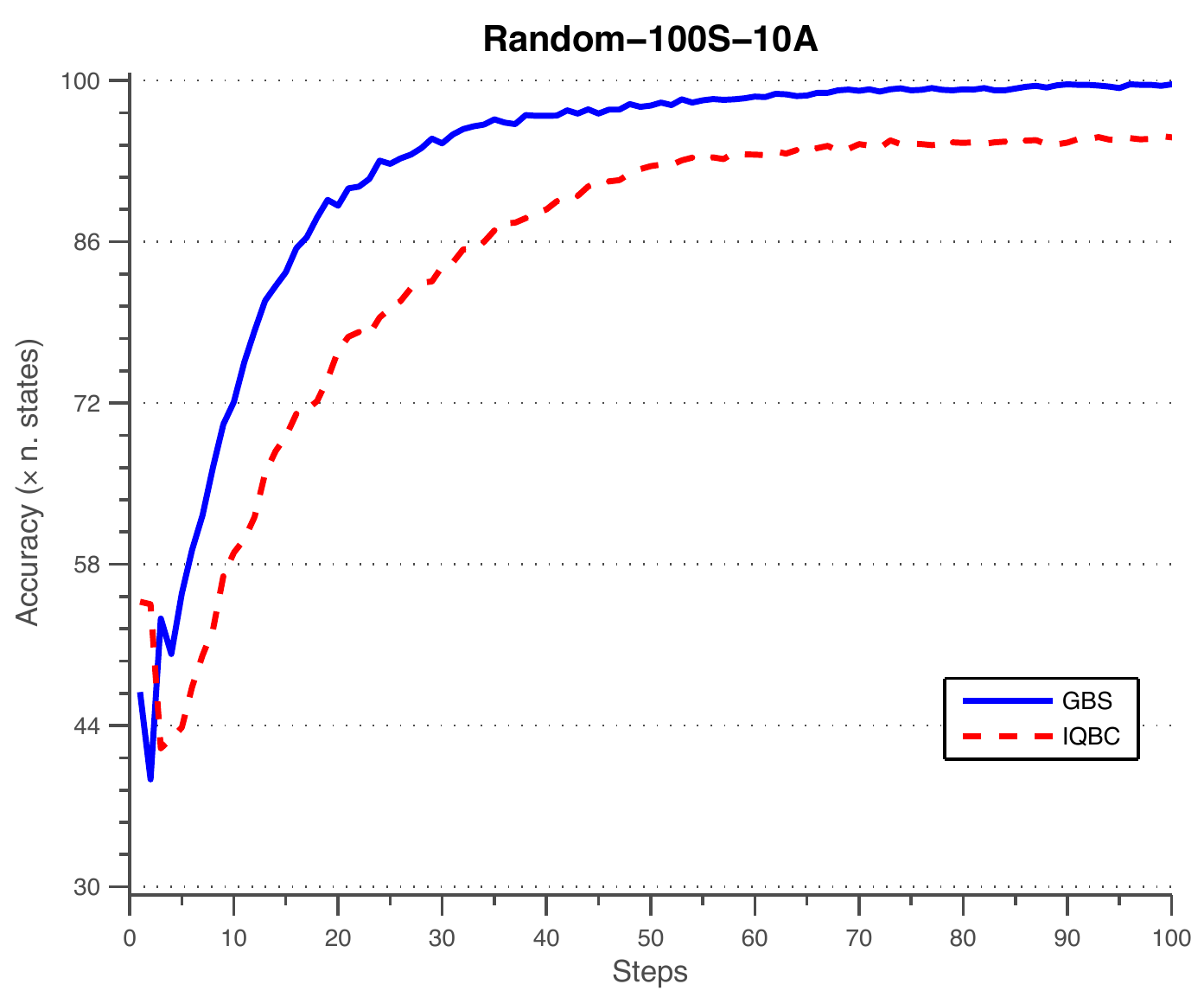}
  }\hfill
  \subfigure[Value perf.\ ($100\times10$)]{\label{Fig:Value-100-10}
    \includegraphics[width=0.4\columnwidth]{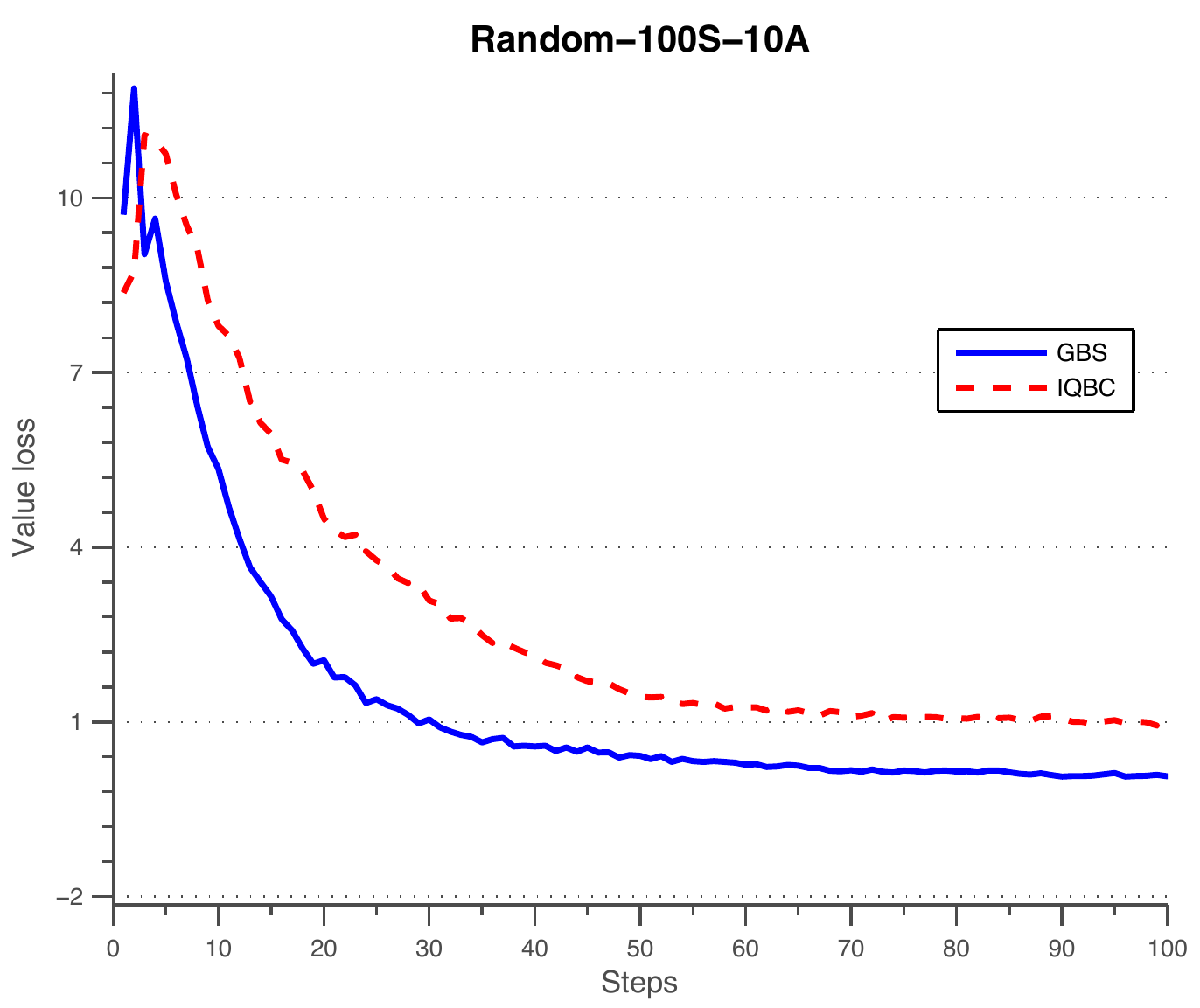}
  }
  \caption{Classification and value performance of GBS-IRL and IQBC in medium-sized random MDPs. Solid lines correspond to GBS-IRL, and dotted lines correspond to IQBC. \subref{Fig:Policy-50-5}-\subref{Fig:Policy-100-10} Classification performance. \subref{Fig:Value-50-5}-\subref{Fig:Value-100-10} Value performance. The indicated values correspond to the dimensions $\abs{\X}\times\abs{\A}$ of the MDPs.}
  \label{Fig:Results-50-100}
\end{figure}

Given the results in the first set of experiments and the computation time already associated with EMG, in the remaining experiments we opted by comparing GBS-IRL with IQBC only. The learning curves in terms both of policy accuracy and task execution are depicted in Fig.~\ref{Fig:Results-50-100}.

In this set of experiments we can observe that the performance of IQBC appears to deteriorate more severely with the number of actions than that of GBS-IRL. Although not significantly, this tendency could already be observed in the smaller environments (see, for example, Fig.~\ref{Fig:Value-10-10}). This dependence on the number of actions is not completely unexpected. In fact, IQBC queries states $x$ that maximize
\begin{displaymath}
VE(x)=-\sum_{a\in\A}\frac{n_{\H}(x,a)}{\abs{\H}}\log\frac{n_{\H}(x,a)}{\abs{\H}},
\end{displaymath}
where $n_{\H}(x,a)$ is the number of hypothesis $\vec{h}\in\H$ such that $a\in\A_{\vec{h}}(x)$. Since the disagreement is taken over the set of all possible actions, there is some dependence of the performance of IQBC on the number of actions. 

GBS-IRL, on the other hand, is more focused toward identifying \emph{one} optimal action per state. This renders our approach less sensitive to the number of actions, as can be seen in Corollaries~\ref{Cor:SampleComplexity} through \ref{Cor:Convergence-2} and illustrated in Fig.~\ref{Fig:Results-50-100}.


\subsubsection*{Large-sized structured domains}

So far, we have analyzed the performance of GBS-IRL in random MDPs with no particular structure, both in terms of transition probabilities and reward function. In the third set of experiments, we look further into the scalability of GBS-IRL by considering large-sized domains. We consider more structured problems selected from the IRL literature. In particular, we evaluate the performance of GBS-IRL in the \emph{trap-world}, \emph{puddle-world} and \emph{driver} domains.

\begin{figure}[!tb]
\begin{minipage}{0.45\columnwidth}
  \centering
  \includegraphics[width=0.67\columnwidth]{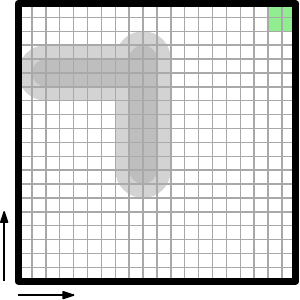}
  \caption{The puddle-world domain \citep{boyan95nips}.}
  \label{Fig:Puddle-world}
\end{minipage}
  \hfill
\begin{minipage}{0.45\columnwidth}
  \centering
  \includegraphics[width=\columnwidth]{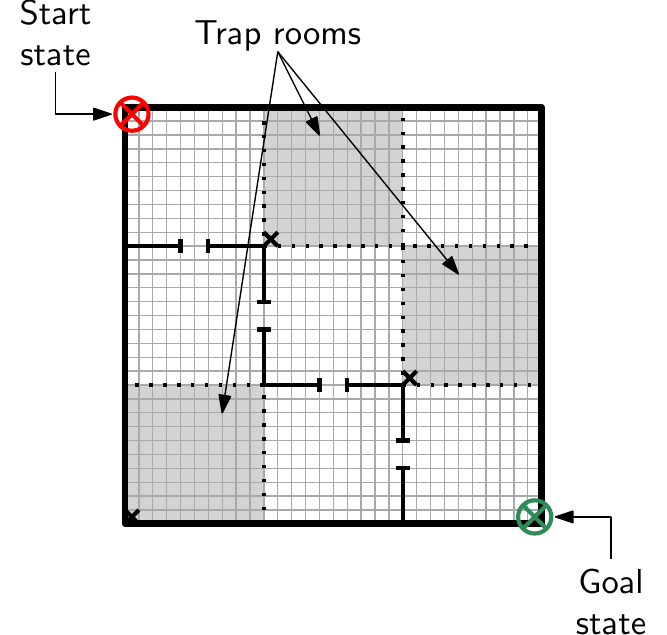}
  \caption{The trap-world domain \citep{judah11icml}.}
  \label{Fig:Trap-world}
\end{minipage}
\end{figure}

The \emph{puddle-world} domain was introduced in the work of \cite{boyan95nips}, and is depicted in Fig.~\ref{Fig:Puddle-world}. It consists of a $20\times20$ grid-world in which two ``puddles'' exist (corresponding to the darker cells). When in the puddle, the agent receives a penalty that is proportional to the squared distance to the nearest edge of the puddle, and ranges between  $0$ and $-1$. The agent must reach the goal state in the top-right corner of the environment, upon which it receives a reward of $+1$. We refer to the original description of \cite{boyan95nips} for further details.

This domain can be described by an MDP with $\abs{\X}=400$ and $\abs{\A}=4$, where the four actions correspond to motion commands in the four possible directions. Transitions are stochastic, and can be described as follows. After selecting the action corresponding to moving in direction $d$, the agent will roll back one cell (\ie move in the direction $-d$) with a probability $0.06$. With a probability $0.24$ the action will fail and the agent will remain in the same position. The agent will move to the adjacent position in direction $d$ with probability $0.4$. With a probability $0.24$ it will move two cells in direction $d$, and with probability $0.06$ it will move three cells in direction $d$. We used a discount $\gamma=0.95$ for the MDP (not to be confused with the noise parameters, $\hgamma(x)$). 

The \emph{trap-world} domain was introduced in the work of \cite{judah11icml}, and is depicted in Fig.~\ref{Fig:Trap-world}. It consists of a $30\times30$ grid-world separated into 9 rooms. Darker rooms correspond to \emph{trap rooms}, from which the agent can only leave by reaching the corresponding bottom-left cell (marked with a ``$\times$''). Dark lines correspond to walls that the agent cannot traverse. Dotted lines are used to delimit the trap-rooms from the safe rooms but are otherwise meaningless. The agent must reach the goal state in the bottom-right corner of the environment. We refer to the work of \cite{judah11icml} for a more detailed description.

This domain can be described by an MDP with $\abs{\X}=900$ and $\abs{\A}=4$, where the four actions correspond to motion commands in the four possible directions. Transitions are deterministic. The target reward function $r^*$ is everywhere 0 except on the goal, where $r^*(x_{\rm goal})=1$. We again used a discount $\gamma=0.95$ for the MDP. 

\begin{figure}[!tb]
\centering
  \includegraphics[width=0.6\columnwidth]{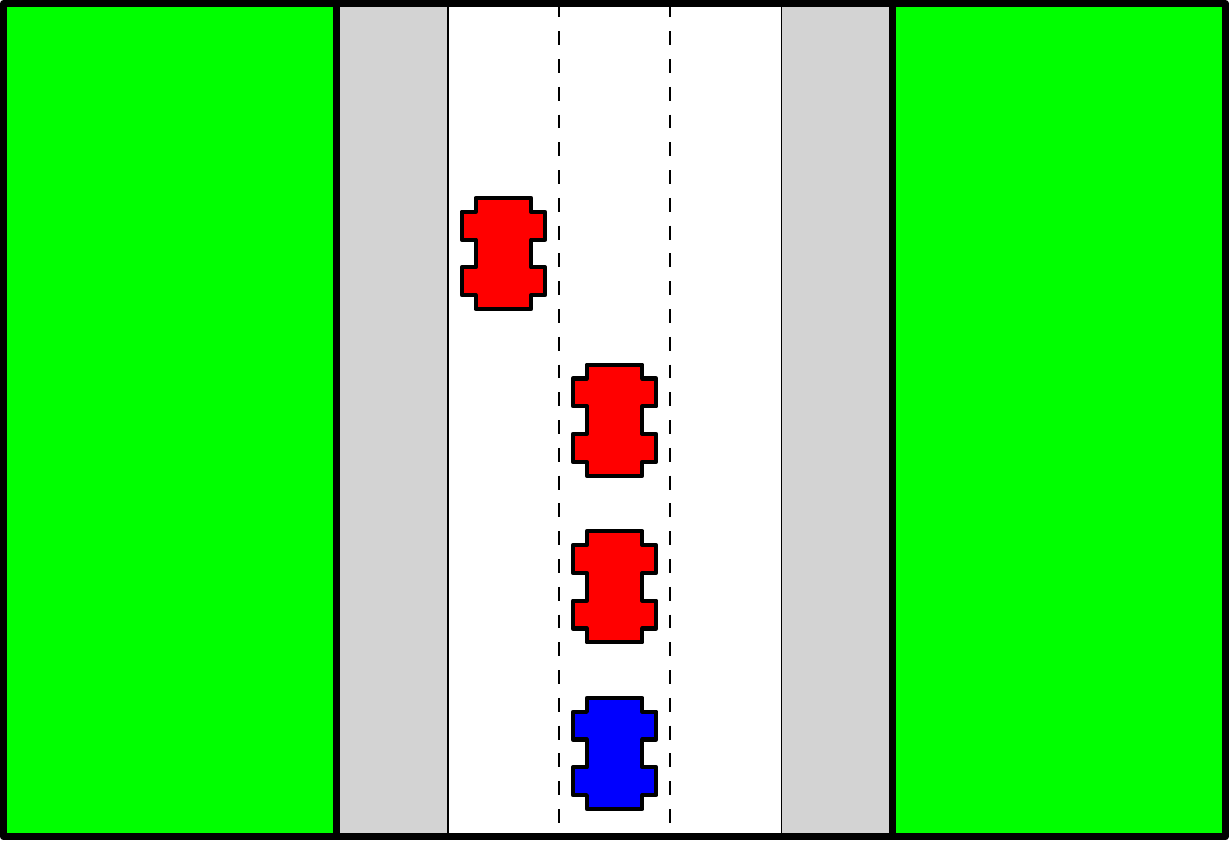}
  \caption{The driver-world domain \citep{abbeel04icml}.}
  \label{Fig:Driver-world}
\end{figure}

Finally, the \emph{driver domain} was introduced in the work of \cite{abbeel04icml}, an instance of which is depicted in Fig.~\ref{Fig:Driver-world}. In this environment, the agent corresponds to the driver of the blue car at the bottom, moving at a speed greater than all other cars. All other cars move at constant speed and are scattered across the three central lanes. The goal of the agent is to drive as safely as possible---\ie avoid crashing into other cars, turning too suddenly and, if possible, driving in the shoulder lanes.

For the purposes of our tests, we represented the driver domain as an MDP with $\abs{\X}=16,875$ and $\abs{\A}=5$, where the five actions correspond to driving the car into each of the 5 lanes. Transitions are deterministic. The target reward function $r^*$ penalizes the agent with a value of $-10$ for every crash, and with a value of $-1$ for driving in the shoulder lanes. Additionally, each lane change costs the agent a penalty of $-0.1$. As in the previous scenarios, we used a discount $\gamma=0.95$ for the MDP.

\begin{figure}[!tb]
\centering
  \subfigure[Policy accur.\ (puddle-w.).]{\label{Fig:Policy-puddle}
    \includegraphics[width=0.4\columnwidth]{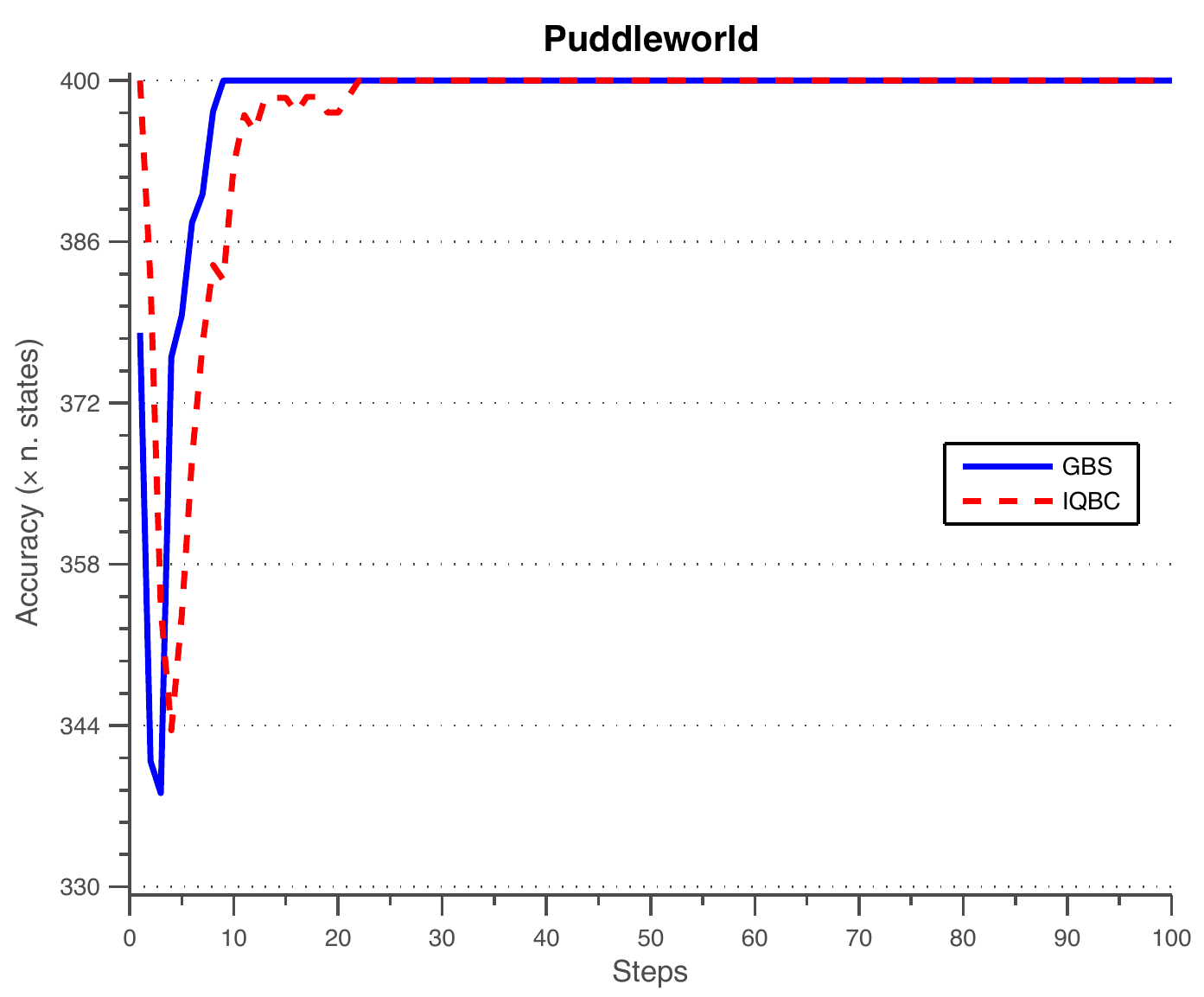}
  }\hfill
  \subfigure[Policy accur.\ (trap-w.).]{\label{Fig:Policy-trap}
    \includegraphics[width=0.4\columnwidth]{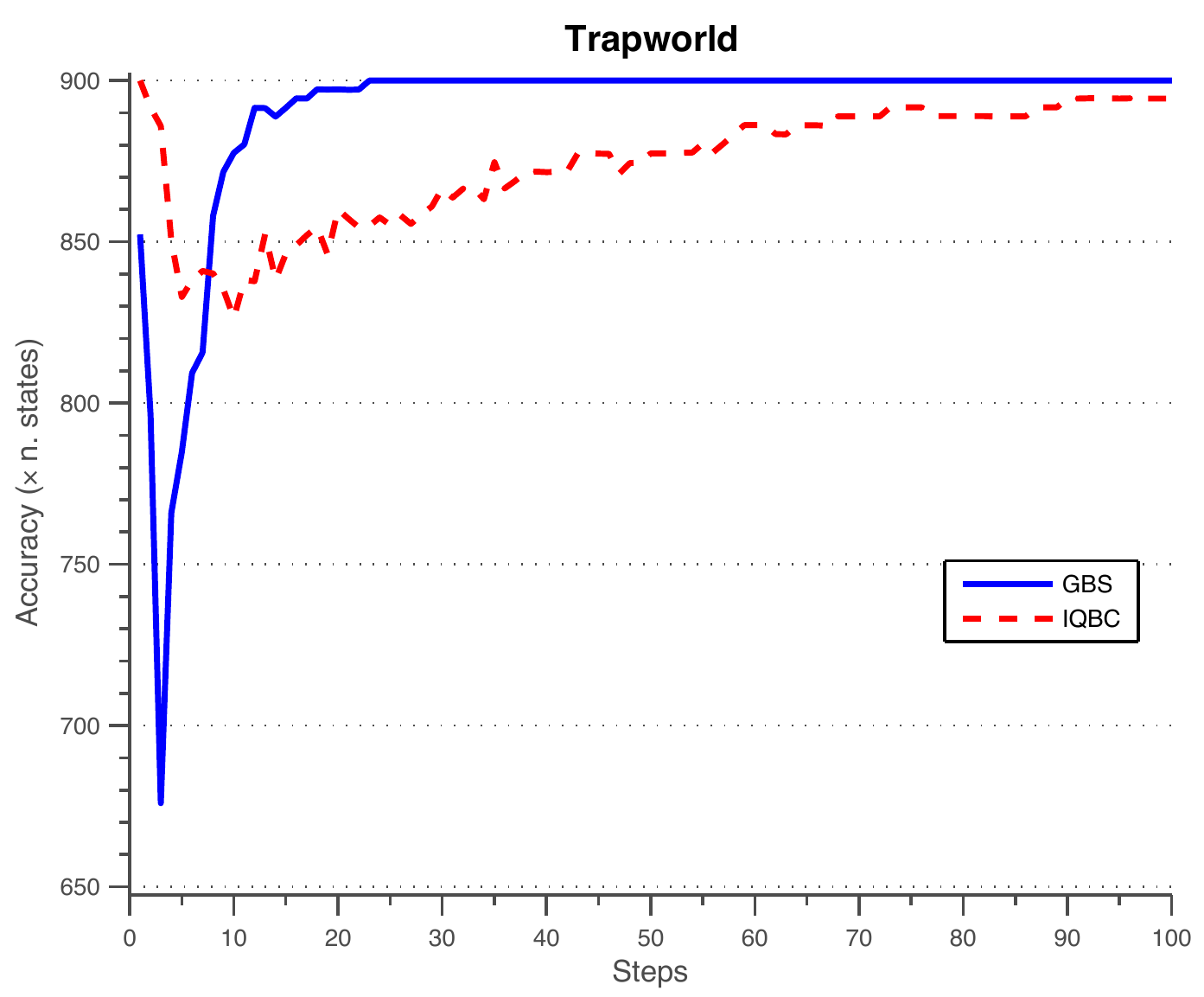}
  }\\
  \subfigure[Policy accur.\ (driver).]{\label{Fig:Policy-driver}
    \includegraphics[width=0.4\columnwidth]{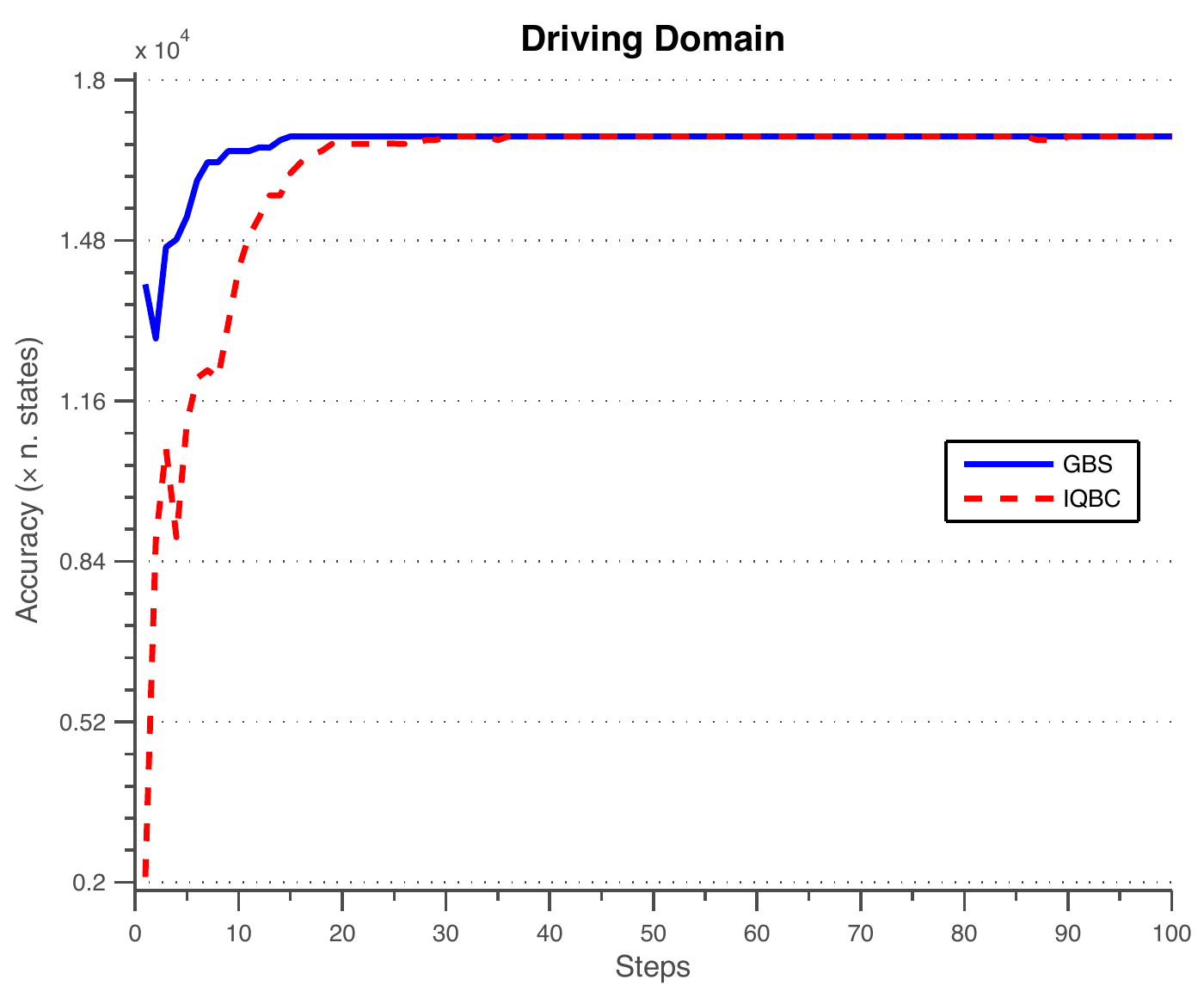}
  }\hfill
  \subfigure[Value perf.\ (puddle-w.).]{\label{Fig:Value-puddle}
    \includegraphics[width=0.4\columnwidth]{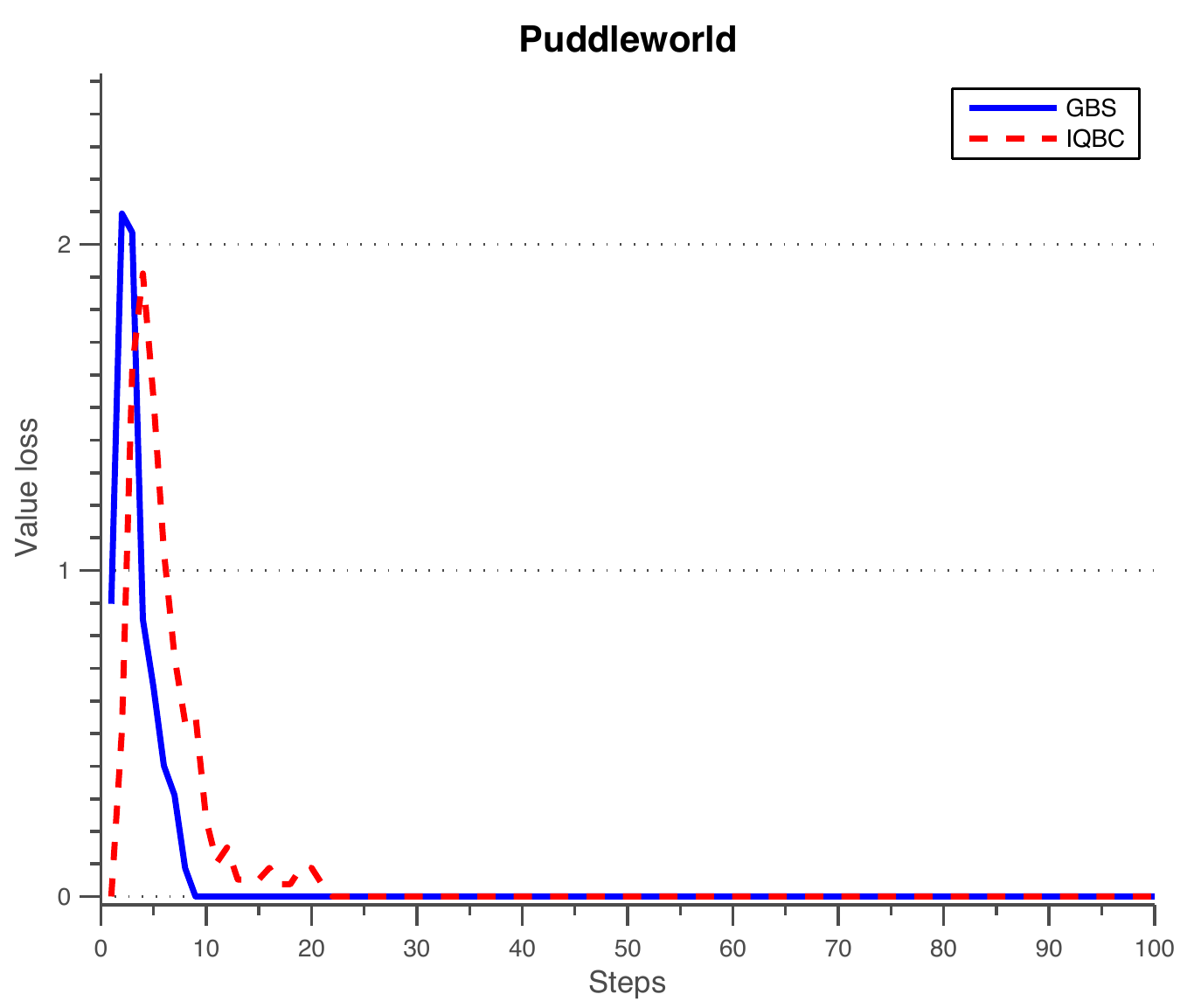}
  }\\
  \subfigure[Value perf.\ (trap-w.).]{\label{Fig:Value-trap}
    \includegraphics[width=0.4\columnwidth]{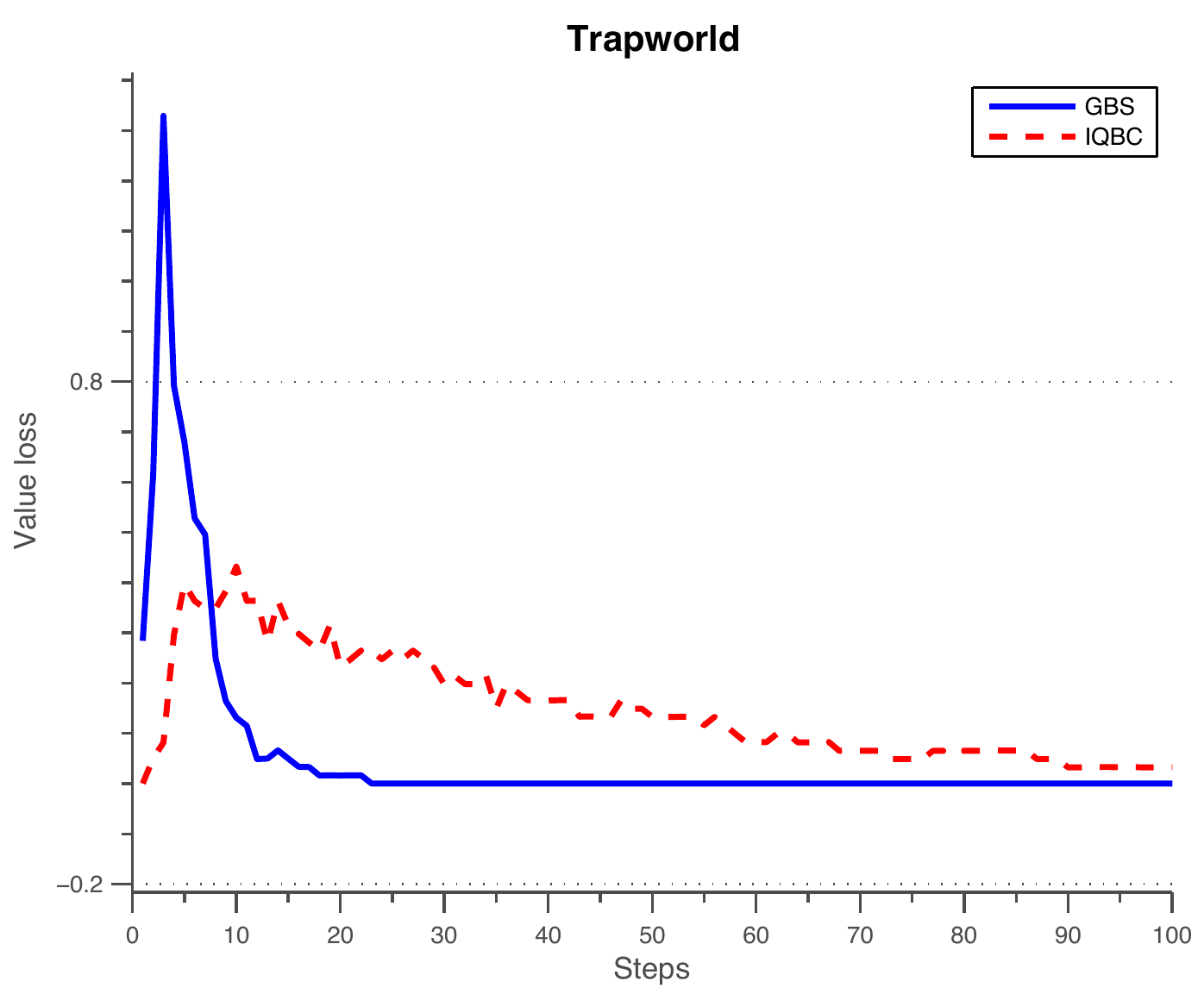}
  }\hfill
  \subfigure[Value perf.\ (driver).]{\label{Fig:Value-driver}
    \includegraphics[width=0.4\columnwidth]{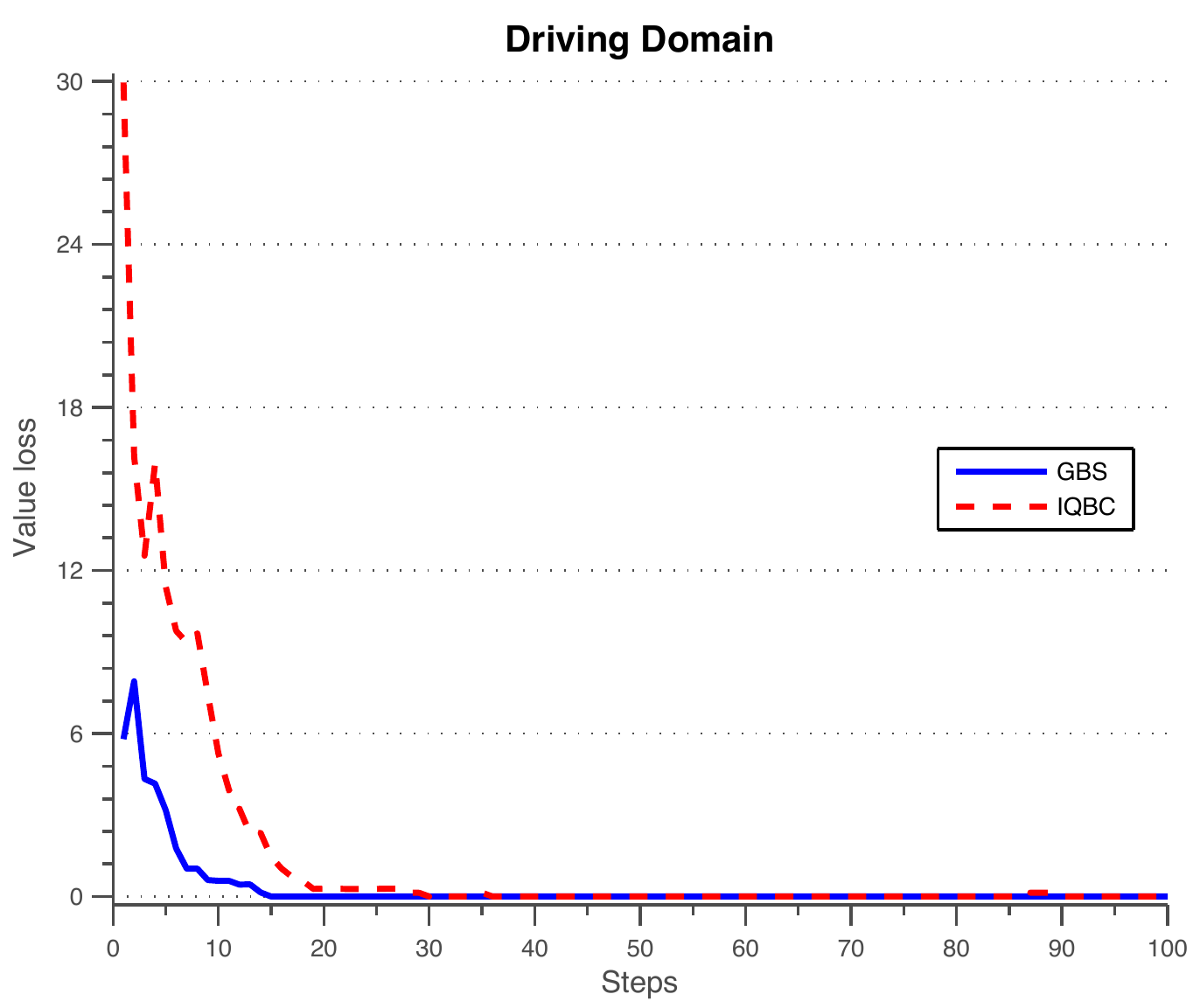}
  }
\caption{Classification and value performance of GBS-IRL and IQBC in the three large domains. Solid lines correspond to GBS-IRL, and dotted lines correspond to IQBC. \subref{Fig:Policy-trap}-\subref{Fig:Policy-driver} Classification performance. \subref{Fig:Value-trap}-\subref{Fig:Value-driver} Value performance.}
  \label{Fig:Results-domains}
\end{figure}

As with the previous experiments, we conducted 200 independent learning trials for each of the three environments, and evaluated the performance of both GBS-IRL and IQBC. The results are depicted in Fig.~\ref{Fig:Results-domains}.

We can observe that, as in previous scenarios, the performance of both methods is very similar. All scenarios feature a relatively small number of actions, which attenuates the negative dependence of IQBC on the number of actions observed in the previous experiments. 

It is also interesting to observe that the trap-world domain seems to be harder to learn than the other two domains, in spite of the differences in dimension. For example, while the driver domain required only around 10 samples for GBS-IRL to single out the correct hypothesis, the trap-world required around 20 to attain a similar performance. This may be due to the fact that the trap-world domain features the sparsest reward. Since the other rewards in the hypothesis space were selected to be similarly sparse, it is possible that many would lead to similar policies in large parts of the state-space, thus hardening the identification of the correct hypothesis.

To conclude, is is still interesting to observe that, in spite of the dimension of the problems considered, both methods were effectively able to single out the correct hypothesis after only a few samples. In fact, the overall performance is superior to that observed in the medium-sized domains, which indicates that the domain structure present in these scenarios greatly contributes to disambiguate between hypothesis, given the expert demonstration.


\subsection{Using Action and Reward Feedback}
\label{Sec:Policy+Reward}

To conclude the empirical validation of our approach, we conduct a final set of experiments that aims at illustrating the applicability of our approach in the presence of both action and reward feedback.

\begin{figure}[!tb]
\centering
  \includegraphics[width=0.4\columnwidth]{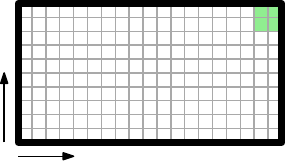}
  \caption{The grid-world used to illustrate the combined use of action and reward feedback.}
  \label{Fig:Grid-world}
\end{figure}

One first experiment illustrates the integration of both reward and policy information in the Bayesian IRL setting described in Section~\ref{Subsec:MultiAction}. We consider the simple $19\times10$ grid-world depicted in Fig.~\ref{Fig:Grid-world}, where the agent must navigate to the top-right corner of the environment. In this first experiment, we use random sampling, in which, at each time step $t$, the expert adds one (randomly selected) sample to the demonstration $\F_t$, which can be of either form $(x_t,a_t)$ or $(x_t,u_t)$. 

\begin{figure}[!tb]
\centering
  \includegraphics[width=0.6\columnwidth]{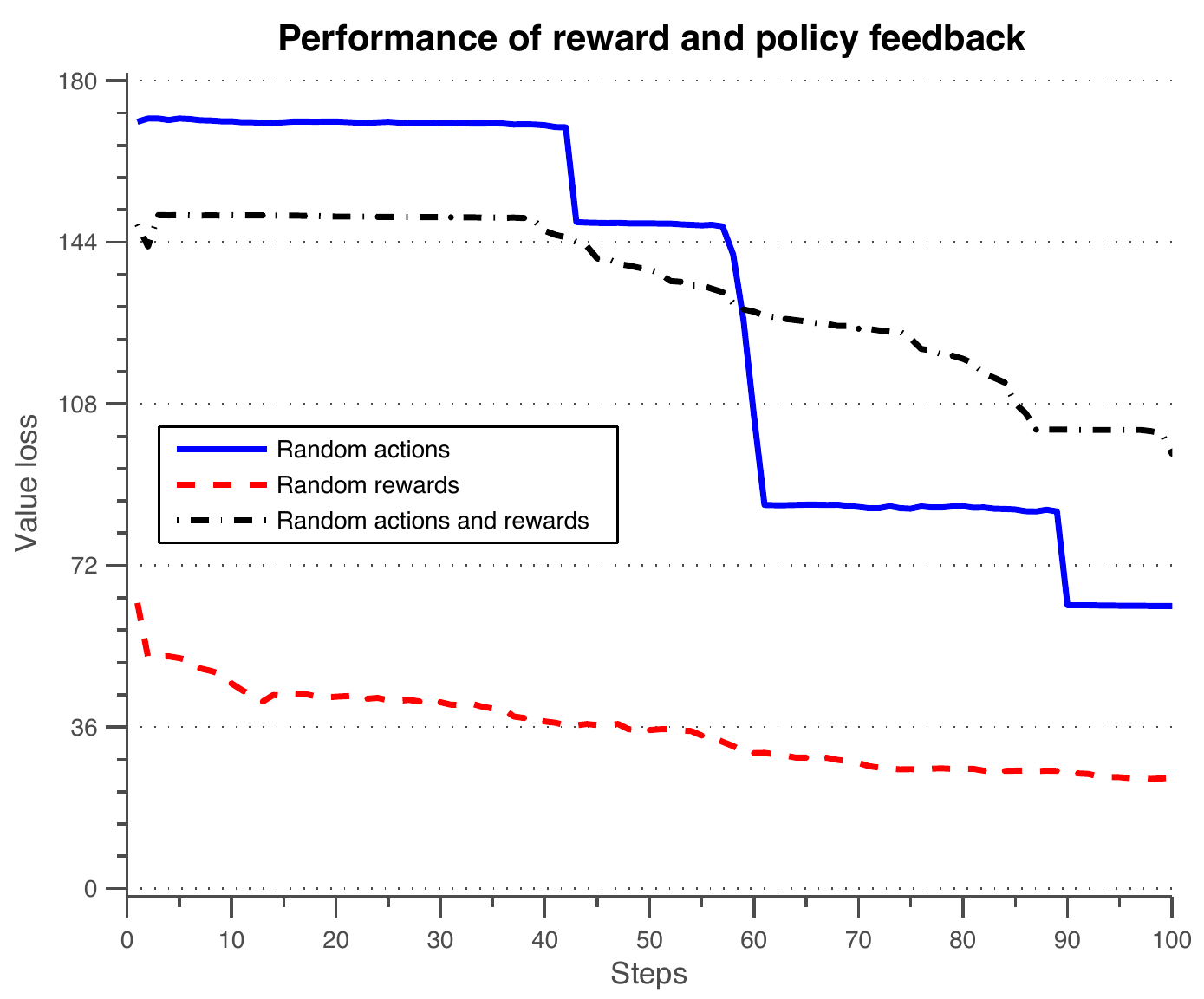}
  \caption{Bayesian IRL using reward and action feedback.}
  \label{Fig:Reward-action}
\end{figure}

Figure~\ref{Fig:Reward-action} compares the performance of Bayesian IRL for demonstrations consisting of state-action pairs only, state-reward pairs only, and also demonstrations that include both state-action and state-reward pairs. 

We first observe that all demonstration types enable the learner to slowly improve its performance in the target task. This indicates that all three sources of information (action, reward, and action+reward) give useful information to accurately identify the target task (or, equivalently, identify the target reward function). 

Another important observation is that a direct comparison between the learning performance obtained with the different demonstration types may be misleading, since the ability of the agent to extract useful information from the reward samples greatly depends on the sparsity of the reward function. Except in those situations in which the reward is extremely informative, an action-based demonstration will generally be more informative.

\begin{figure}[!tb]
\centering
  \subfigure[]{\label{Fig:Reward-sparse}
    \includegraphics[width=0.45\columnwidth]{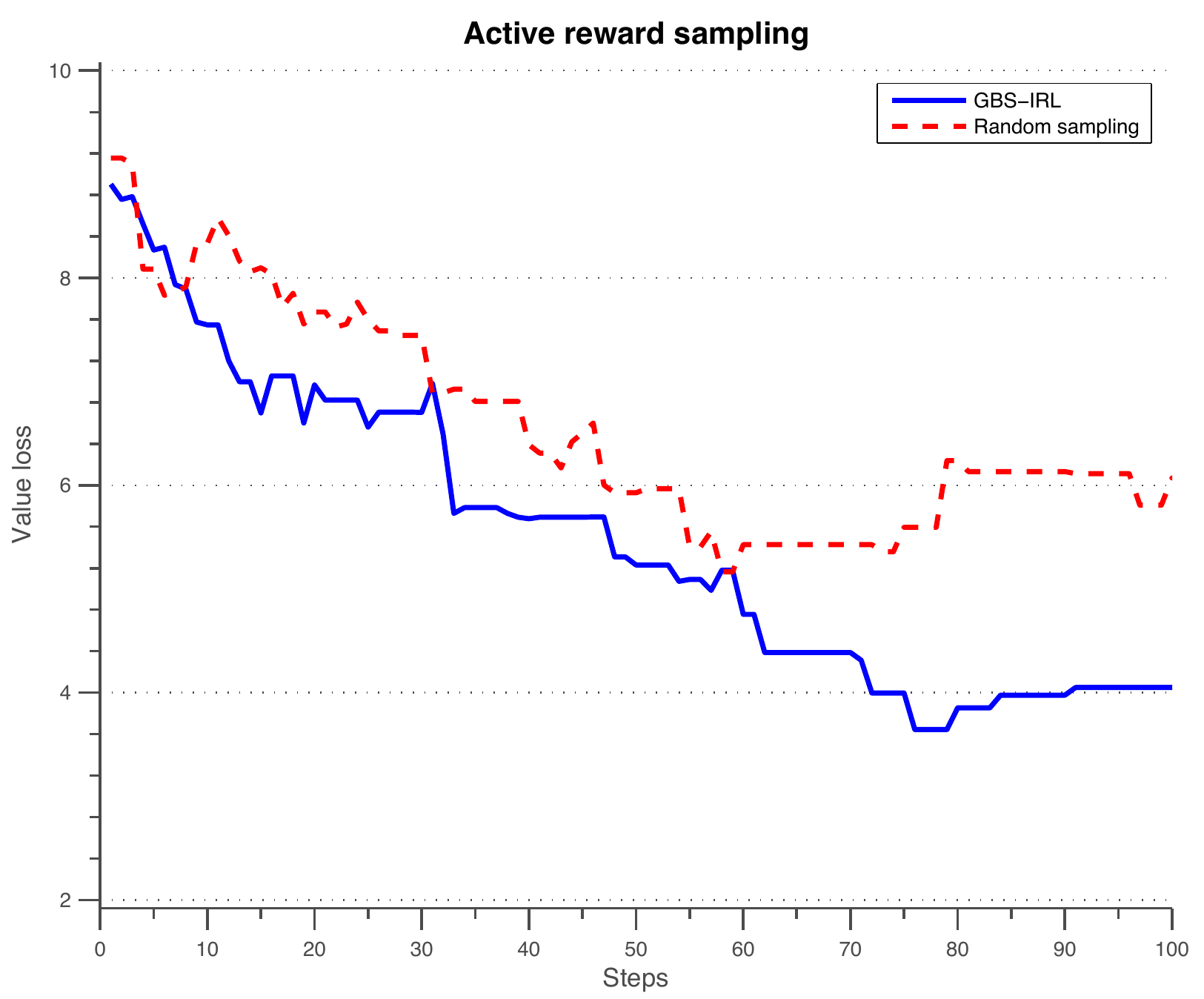}
  }\hfill
  \subfigure[]{\label{Fig:Reward-dense}
    \includegraphics[width=0.45\columnwidth]{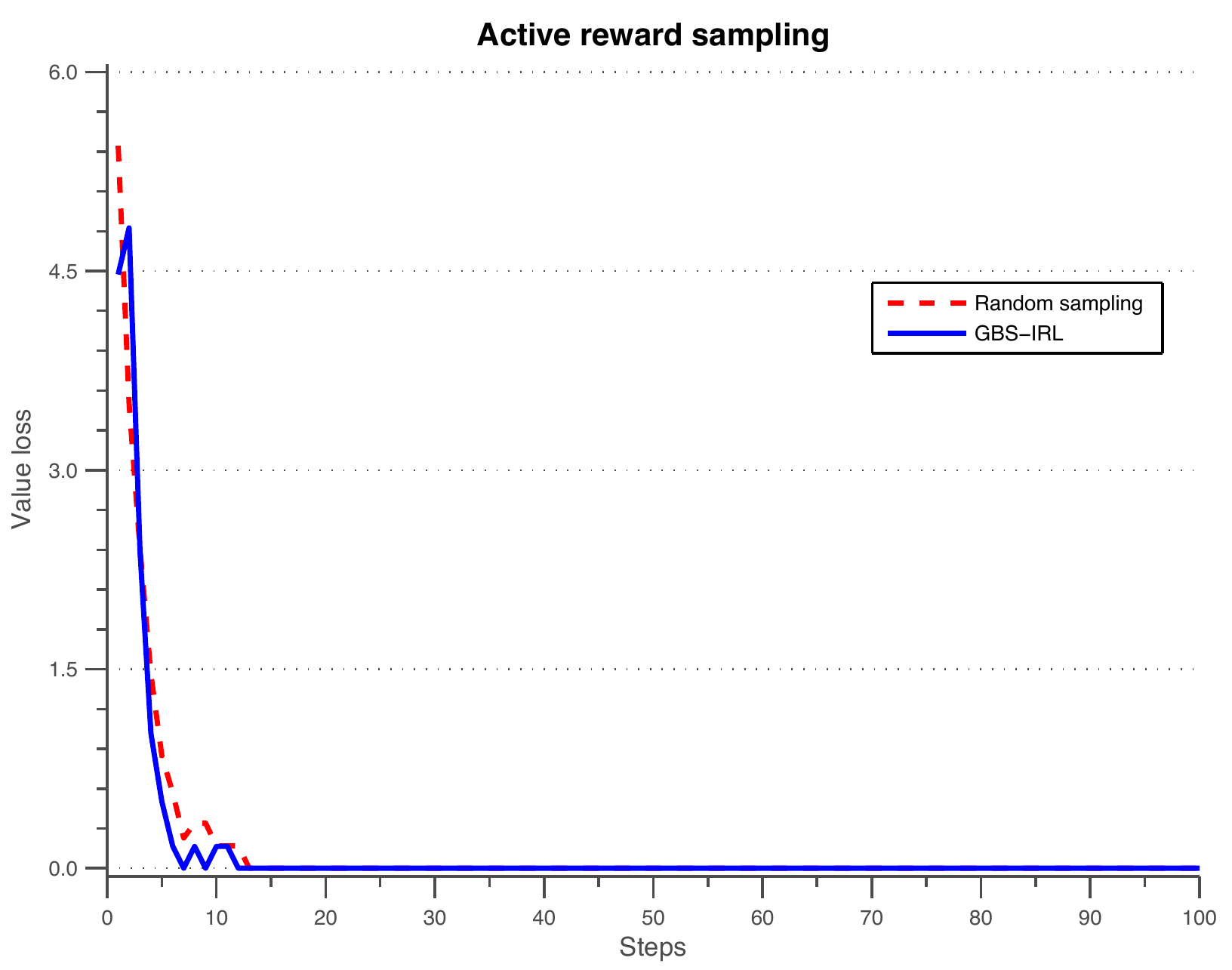}
  }
  \caption{Active IRL using reward feedback: sparse vs dense rewards.}
  \label{Fig:Rewards}
\end{figure}

In a second experiment, we analyze the performance of our active learning method when querying only reward information in the same grid-world environment. In particular, we analyze the dependence of the performance on the sparsity of the reward function, testing GBS-IRL in two distinct conditions. The first condition, depicted in Fig.~\ref{Fig:Reward-sparse}, corresponds to a reward function $r^*$ that is sparse, \ie such that $r^*(x)=0$ for all states $x$ except the goal states, where $r^*(x_{\rm goal})=1$. 

As discussed in Section~\ref{Subsec:MultiAction}, sparsity of rewards greatly impacts the learning performance of our Bayesian IRL approach. This phenomenon, however, is not exclusive to the active learning approach---in fact, as seen from Fig.~\ref{Fig:Reward-sparse}, random sampling also exhibits a poor performance. It is still possible, nonetheless, to detect some advantage in using an active sampling approach. 

In contrast, it is possible to design very informative rewards, by resorting to a technique proposed in the reinforcement learning literature under the designation of \emph{reward shaping} \citep{ng99icml}. By considering a shaped version of that same reward, we obtain the learning performance depicted in Fig.~\ref{Fig:Reward-dense}. Note how, in the latter case, convergence is extremely fast even in the presence of random sampling. 

We conclude by noting that, in the case of reward information, our setting is essentially equivalent to a standard reinforcement learning setting, for which efficient exploration techniques have been proposed and may provide fruitful avenues for future research.


\section{Discussion}
\label{Sec:Conclusions}


In this paper we introduce GBS-IRL, a novel active IRL algorithm that allows an agent to learn a task from a demonstration by an ``expert''. Using a generalization of binary search, our algorithm greedily queries the expert for demonstrations in highly informative states. As seen in Section~\ref{Subsec:Related}, and following the designation of \cite{dasgupta11tcs}, GBS-IRL is an \emph{aggressive} active learning algorithm. Additionally, given our consideration of noisy samples, GBS-IRL is naturally designed to consider \emph{non-separable data}. As pointed out by \cite{dasgupta11tcs}, few aggressive active learning algorithms exist with provable complexity bounds for the non-separable case. GBS-IRL comes with such guarantees, summarized in Corollary~\ref{Cor:SampleComplexity}: under suitable conditions and for any given $\delta>0$, $\PP{\vec{h}_t\neq\vec{h}^*}>1-\delta$, as long as
\begin{displaymath}
t\geq\frac{1}{\lambda}\log\frac{\abs{\H}}{\delta},
\end{displaymath}
where $\lambda$ is a constant that does not depend on the dimension of the hypothesis space but only on the sample noise. 

Additionally, as briefly remarked in Section~\ref{Subsec:SingleAction}, it is possible to use an adaptive sub-modularity argument to establish the near-optimality of GBS-IRL. In fact, given the target hypothesis, $\vec{h}^*$, consider the objective function
\begin{displaymath}
f(\F_t)=\PP{h_t\neq h^*\mid\F_t}=1-p_t(h^*).
\end{displaymath}
From Theorem~\ref{Theo:Consistency} and its proof, it can be shown that $f$ is \emph{strongly adaptive monotone} and \emph{adaptive sub modular} and use results of \cite{golovin11jair} to provide a similar bound on sample complexity of GBS-IRL. To our knowledge, GBS-IRL is the first active IRL algorithm with provable sample complexity bounds. Additionally, as discussed in Section~\ref{Subsec:SingleAction}, our reduction of IRL to a standard (multi-class) classification problem implies that Algorithm~\ref{Alg:ActIRL} is not specialized in any particular way to IRL problems. In particular, our results are generally applicable in any multi-class classification problems verifying the corresponding assumptions.

Finally, our main contributions are focused in the simplest form of interaction, when the demonstration consist of examples of the right action to take in different situations. However, we also discuss how other forms of expert feedback (beyond policy information) may be integrated in a seamless manner in our GBS-IRL framework. In particular, we discussed how to combine both policy and reward information in our learning algorithm. Our approach thus provides an interesting bridge between reinforcement learning (or learning by trial and error) and imitation learning (or learning from demonstration). In particular, it brings to the forefront existing results on efficient exploration in reinforcement learning \citep{jaksch10jmlr}. 

Additionally, the general Bayesian IRL framework used in this paper is also amenable to the integration of additional information sources. For example, the human agent may provide trajectory information, or indicate states that are frequently visited when following the optimal path. From the MDP parameters it is generally possible to associate a likelihood with such feedback, which can in turn be integrated in the Bayesian task estimation setting. However, extending the active learning approach to such sources of information is less straightforward and is left as an important avenue for future research.


\appendix

\section{Proofs}%
\label{App:Proofs}

In this appendix we collect the proofs of all statements throughout the paper.


\subsection{Proof of Lemma~\ref{Lemma:Incoherence}}%
\label{Proof:Incoherence}

The method of proof is related to that of \cite{nowak11tit}. We want to show that either
\begin{itemize}
\item $W(p,\x_i)\leq c^*$ for some $\X_i\in\Xi$ or, alternatively, 
\item There are two $k$-neighbor sets $\X_i,\X_j\in\Xi$ such that $W(p,\x_i)>c^*$ and $W(p,\x_j)>c^*$, while $A^*(p,\x_i)\neq A^*(p,\x_j)$. 
\end{itemize}
We have that, for any $a\in\A$,
\begin{displaymath}
\sum_{\vec{h}\in\H}p(\vec{h})\sum_{i=1}^Nh(\x_i,a)\mu(\X_i)\leq\sum_{\vec{h}\in\H}p(\vec{h})c^*=c^*
\end{displaymath}
The above expression can be written equivalently as
\begin{equation}\label{Eq:Incoherence}
\EE[\mu]{\sum_{\vec{h}\in\H}p(\vec{h})h(\x_i,a)}\leq c^*.
\end{equation}
Suppose that there is no $x\in\X$ such that $W(p,x)\leq c^*$. In other words, suppose that, for every $x\in\X$, $W(p,x)>c^*$. Then, for \eqref{Eq:Incoherence} to hold, there must be $\X_i,\X_j\in\Xi$ and $a\in\A$ such that 
\begin{align*}
  \sum_{\vec{h}\in\H}p(\vec{h})h(\x_i,a)&>c^*\\
  \sum_{\vec{h}\in\H}p(\vec{h})h(\x_j,a)&<-c^*.
\end{align*}
Since $(\X,\H)$ is $k$-neighborly by assumption, there is a sequence $\set{\X_{k_1},\ldots,\X_{k_\ell}}$ such that $\X_{k_1}=\X_i$, $\X_{k_\ell}=\X_j$, and  every two sets $\X_{k_m},\X_{k_{m+1}}$ are $k$-neighborly. Additionally, at some point in this sequence, the signal of $\sum_{\vec{h}\in\H}p(\vec{h})h(\x_i,a)$ must change. This implies that there are two $k$-neighboring sets $\X_{k_i}$ and $\X_{k_j}$ such that
\begin{align*}
  \sum_{\vec{h}\in\H}p(\vec{h})h(\x_{k_i},a)&>c^* &
  \sum_{\vec{h}\in\H}p(\vec{h})h(\x_{k_j},a)&<-c^*,
\end{align*}
which implies that
\begin{displaymath}
  A^*(p_t,\x_{k_i})\neq A^*(p_t,\x_{k_j}),
\end{displaymath}
and the proof is complete.


\subsection{Proof of Theorem~\ref{Theo:Consistency}}%
\label{Proof:Consistency}

Let $C_t$ denote the amount of probability mass placed on incorrect hypothesis by $p_t$, i.e.,
\begin{displaymath}
C_t=\frac{1-p_t(\vec{h}^*)}{p_t(\vec{h}^*)}.
\end{displaymath}
The proof of Theorem~\ref{Theo:Consistency} relies on the following fundamental lemma, whose proof can be found in Appendix~\ref{Proof:Supermartingale}.

\begin{lemma}\label{Lemma:Supermartingale}
Under the conditions of Theorem~\ref{Theo:Consistency}, the process $\set{C_t,t=1,\ldots}$ is a non-negative supermartingale with respect to the filtration $\set{\F_t,t=1,\ldots}$. In other words,
\begin{displaymath}
\EE{C_{t+1}\mid\F_t}\leq C_t,
\end{displaymath}
for all $t\geq0$.
\end{lemma}

The proof now replicates the steps in the proof of Theorem~3 of \cite{nowak11tit}. In order to keep the paper as self-contained as possible, we repeat those steps here. We have that
\begin{displaymath}
\mathbb{P}\big[\hat{\vec{h}}_t\neq\vec{h}^*\big]\leq\PP{p_t(\vec{h}^*)<1/2}=\PP{C_t>1}\leq\EE{C_t},
\end{displaymath}
where the last inequality follows from the Markov inequality. Explicit computations yield
\begin{align*}
\mathbb{P}\big[\hat{\vec{h}}_t\neq\vec{h}^*\big]
  &\leq\EE{C_t}\\
  &=\EE{\frac{C_t}{C_{t-1}}C_{t-1}}\\
  &=\EE{\EE{\frac{C_t}{C_{t-1}}C_{t-1}\mid\F_{t-1}}}\\
  &=\EE{\EE{\frac{C_t}{C_{t-1}}\mid\F_{t-1}}C_{t-1}}\\
  &\leq\max_{\F_{t-1}}\EE{\frac{C_t}{C_{t-1}}\mid\F_{t-1}}\EE{C_{t-1}}.
\end{align*}
Finally, expanding the recursion, 
\begin{equation}\label{Eq:Bound}
\mathbb{P}\big[\hat{\vec{h}}_t\neq\vec{h}^*\big]
  \leq C_0\left(\max_{\tau=1,\ldots,t-1}\EE{\frac{C_\tau}{C_{\tau-1}}\mid\F_{\tau-1}}\right)^t.
\end{equation}
Since, from Lemma~\ref{Lemma:Supermartingale}, $\EE{\frac{C_t}{C_{t-1}}\mid\F_{t-1}}<1$ for all $t$, the conclusion follows.


\subsection{Proof of Lemma~\ref{Lemma:Supermartingale}}%
\label{Proof:Supermartingale}

The structure of the proof is similar to that of the proof of Lemma 2 of \cite{nowak11tit}. We start by explicitly writing the expression for the Bayesian update in \eqref{Eq:Bayes}. For all $a\in\A$, let
\begin{eqnarray}\label{Eq:deltaA}
\lefteqn{\delta(a)\triangleq\PP[p_t]{h(x_{t+1},a)=1}=}&\\\nonumber
&=\frac{1}{2}\left(1+\sum_{\vec{h}\in\H}p_t(\vec{h})h(x_{t+1},a)\right),
\end{eqnarray}
and we abusively write $\delta_{t+1}$ to denote $\delta(a_{t+1})$. The quantity $\delta(a)$ corresponds to the fraction of probability mass concentrated on hypotheses prescribing action $a$ as optimal in state $x_{t+1}$. The normalizing factor in the update \eqref{Eq:Bayes} is given by
\begin{align*}
\lefteqn{\sum_{\vec{h}\in\H}p_t(\vec{h})\hgamma(x_{t+1})^{(1+h_{t+1})/2}\hbeta(x_{t+1})^{(1-h_{t+1})/2}}\\
  &\qquad=\sum_{\vec{h}:h_{t+1}=1}p_t(\vec{h})\hgamma(x_{t+1})+\sum_{\vec{h}:h_{t+1}=-1}p_t(\vec{h})\hbeta(x_{t+1})\\
  &\qquad=\delta_{t+1}\hgamma(x_{t+1})+(1-\delta_{t+1})\hbeta(x_{t+1}).
\end{align*}
We can now write the Bayesian update of $p_t(\vec{h})$ as
\begin{equation}\label{Eq:Bayes-2}
p_{t+1}(\vec{h})=p_t(\vec{h})\frac{\hgamma(x_{t+1})^{(1+h_{t+1})/2}\hbeta(x_{t+1})^{(1-h_{t+1})/2}}{\delta_{t+1}\hgamma(x_{t+1})+(1-\delta_{t+1})\hbeta(x_{t+1})}.
\end{equation}
Let
\begin{equation}\label{Eq:Eta}
\eta(a)=\frac{\delta(a)\hgamma(x_{t+1})+(1-\delta(a))\hbeta(x_{t+1})}{\hgamma(x_{t+1})^{(1+h^*(x_{t+1},a))/2}\hbeta(x_{t+1})^{(1-h^*(x_{t+1},a))/2}},
\end{equation}
where, as with $\delta$, we abusively write $\eta_{t+1}$ to denote $\eta(a_{t+1})$. Then, for $\vec{h}^*$, we can now write the update~\eqref{Eq:Bayes-2} simply as $p_{t+1}(\vec{h}^*)=p_t(\vec{h}^*)/\eta_{t+1}$, and 
\begin{displaymath}
\frac{C_{t+1}}{C_t}
  =\frac{(1-p_t(\vec{h}^*)/\eta_{t+1})\eta_{t+1}}{1-p_t(\vec{h}^*)}
  =\frac{\eta_{t+1}-p_t(\vec{h}^*)}{1-p_t(\vec{h}^*)}
\end{displaymath}
The conclusion of the Lemma holds as long as $\EE{\eta_{t+1}\mid\F_t}\leq1$. Conditioning the expectation $\EE{\eta_{t+1}\mid\F_t}$ on $x_{t+1}$, we have that
\begin{align*}
\lefteqn{\EE{\eta_{t+1}\mid\F_t,x_{t+1}}}\\
&\qquad=\sum_{a\in\A}\eta(a)\PP{a_{t+1}=a\mid\F_t,x_{t+1}}\\
&\qquad=\sum_{a\in\A}\eta(a)\gamma^*(x_{t+1})^{(1+h^*(x_{t+1},a))/2}\beta^*(x_{t+1})^{(1-h^*(x_{t+1},a))/2}.
\end{align*}
Let $\ah$ denote the action in $\A$ such that $h^*(x_{t+1},\ah)=1$. This leads to
\begin{equation}\label{Eq:Bracket}
\EE{\eta_{t+1}\mid\F_t,x_{t+1}}=\eta(\ah)\gamma^*(x_{t+1})+\sum_{a\neq\ah}\eta(a)\beta^*(x_{t+1}).
\end{equation}
For simplicity of notation, we temporarily drop the explicit dependence of $\beta^*$, $\hbeta$, $\gamma^*$ and $\hgamma$ on $x_{t+1}$. Explicit computations now yield
\begin{align*}
\lefteqn{\eta(\ah)\gamma^*(x_{t+1})+\sum_{a\neq\ah}\eta(a)\beta^*(x_{t+1})}\\
  &\qquad=[\delta(\ah)\hgamma+(1-\delta(\ah))\hbeta]\frac{\gamma^*}{\hgamma}+\\
  &+\qquad\sum_{a\neq\ah}[\delta(a)\hgamma+(1-\delta(a))\hbeta]\frac{\beta^*}{\hbeta}\\
  &\qquad=\delta(\ah)\gamma^*+(1-\delta(\ah))\frac{\hbeta\gamma^*}{\hgamma}+\\
  &+\qquad\sum_{a\neq\ah}\left[\delta(a)\frac{\hgamma\beta^*}{\hbeta}+(1-\delta(a))\beta^*\right].
\end{align*}
Since $\sum_{a\neq\ah}\delta(a)=1-\delta(\ah)$,
\begin{align*}
\lefteqn{\eta(\ah)\gamma^*(x_{t+1})+\sum_{a\neq\ah}\eta(a)\beta^*(x_{t+1})}\\
  &\qquad=(1-\delta(\ah))\left[\frac{\hbeta\gamma^*}{\hgamma}+\frac{\hgamma\beta^*}{\hbeta}\right]+\\
  &\qquad+\delta(\ah)\gamma^*+\sum_{a\neq\ah}(1-\delta(a))\beta^*\\
  &\qquad=(1-\delta(\ah))\left[\frac{\hbeta\gamma^*}{\hgamma}+\frac{\hgamma\beta^*}{\hbeta}\right]+\\
  &\qquad+\delta(\ah)\gamma^*+(\abs{\A}-1)\beta^*-(1-\delta(\ah))\beta^*\\
  &\qquad=(1-\delta(\ah))\left[\frac{\hbeta\gamma^*}{\hgamma}+\frac{\hgamma\beta^*}{\hbeta}\right]+\\
  &\qquad+\delta(\ah)(\gamma^*+\beta^*)+1-\gamma^*-\beta^*,
\end{align*}
where we have used the fact that $(\abs{A}-1)\beta^*+\gamma^*=1$. Finally, we have
\begin{align*}
\lefteqn{\eta(\ah)\gamma^*(x_{t+1})+\sum_{a\neq\ah}\eta(a)\beta^*(x_{t+1})}\\
  &=(1-\delta(\ah))\left[\frac{\hbeta\gamma^*}{\hgamma}+\frac{\hgamma\beta^*}{\hbeta}-\gamma^*-\beta^*\right]+1\\
  &=1-(1-\delta(\ah))\left[\gamma^*+\beta^*-\frac{\hbeta\gamma^*}{\hgamma}-\frac{\hgamma\beta^*}{\hbeta}\right]\\
  &=1-(1-\delta(\ah))\left[\gamma^*\frac{\hgamma-\hbeta}{\hgamma}+\beta^*\frac{\hbeta-\hgamma}{\hbeta}\right].
\end{align*}
Letting $\rho=1-(\abs{\A}-1)\alpha$, we have that $\EE{\eta_{t+1}\mid\F_t,x_{t+1}}<1$ as long as 
\begin{displaymath}
\gamma^*(x)\frac{\hgamma(x)-\hbeta(x)}{\hgamma(x)}+\beta^*(x)\frac{\hbeta(x)-\hgamma(x)}{\hbeta(x)}>0
\end{displaymath}
for all $x\in\X$. Since, for all $x\in\X$, $\beta^*(x)<\alpha$ and $\gamma^*(x)>\rho$, we have
\begin{align*}
\lefteqn{\gamma^*(x)\frac{\hgamma(x)-\hbeta(x)}{\hgamma(x)}+\beta^*(x)\frac{\hbeta(x)-\hgamma(x)}{\hbeta(x)}=}\\
    &=(\gamma^*(x)-\beta^*(x))\left[\frac{\gamma^*(x)}{\hgamma(x)}-\frac{\beta^*(x)}{\hbeta(x)}\right]\\
    &>(\rho-\alpha)\left[\frac{\rho}{\hgamma(x)}-\frac{\alpha}{\hbeta(x)}\right]\geq0.
\end{align*}
where the inequality is strict if $\hbeta(x)>\alpha$ for all $x\in\X$.


\subsection{Proof of Theorem~\ref{Theo:Convergence}}%
\label{Proof:Convergence}

To prove Theorem~\ref{Theo:Convergence}, we depart from \eqref{Eq:Bound}:
\begin{displaymath}
\mathbb{P}\big[\hat{\vec{h}}_t\neq\vec{h}^*\big]
  \leq C_0\left(\max_{\tau=1,\ldots,t-1}\EE{\frac{C_\tau}{C_{\tau-1}}\mid\F_{\tau-1}}\right)^t.
\end{displaymath}
Letting
\begin{displaymath}
\lambda_t=\max_{\tau=1,\ldots,t-1}\EE{\frac{C_\tau}{C_{\tau-1}}\mid\F_{\tau-1}},
\end{displaymath}
the desired result can be obtained by bounding the sequence $\set{\lambda_t,t=0,\ldots}$ by some value $\lambda<1$. To show that such $\lambda$ exists, we consider separately the two possible queries in Algorithm~\ref{Alg:ActIRL}. 

Let then $c_t=\min_{i=1,\ldots,N}W(p_t,\x_i)$, and suppose that there are no 1-neighbor sets $\X_i$ and $\X_j$ such that     
\begin{align}\label{Eq:NeighborSets}
  W(p_t,\x_i)&>c_t,&
  W(p_t,\x_j)&>c_t,\\
  A^*(p_t,\x_i)&\neq A^*(p_t,\x_j).
\end{align}
Then, from Algorithm~\ref{Alg:ActIRL}, the queried state $x_{t+1}$ will be such that 
\begin{displaymath}
x_{t+1}\in\argmin_iW(p_t,\x_i).
\end{displaymath}
Since, from the definition of $c^*$, $c_t<c^*$, it follows that $\delta(a)\leq\frac{1+c^*}{2}$ for all $a\in\A$, where $\delta(a)$ is defined in \eqref{Eq:deltaA}. Then, from the proof of Lemma~\ref{Lemma:Supermartingale}, 
\begin{align*}
\EE{\eta_{t+1}\mid\F_t,x_{t+1}}
  &\leq1-\eps(1-\delta(\ah))\\
  &\leq1-\eps\frac{1-c^*}{2},
\end{align*}
where $\ah$ denotes the action in $\A$ such that $h^*(x,\ah)=1$. 

Consider now the case where there are $1$-neighboring sets $\X_i$ and $\X_j$ such that \eqref{Eq:NeighborSets} holds. In this case, according to Algorithm~\ref{Alg:ActIRL}, $x_{t+1}$ is selected randomly as either $[x]_i$ or $[x]_j$ with probability $1/2$. Moreover, since $\X_i$ and $\X_j$ are 1-neighbors, there is a single hypothesis, say $\vec{h}_0$, that prescribes different optimal actions in $\X_i$ and $\X_j$. Let $\ah_i$ denote the optimal action at $\x_i$, and $\ah_j$ the optimal action at $\x_j$, as prescribed by $\vec{h}^*$. Three situations are possible:
\begin{enumerate}[\bf {Situation~}1.]
\item\label{Enum:Case1} $A^*(p_t,\x_i)\neq\ah_i$ and $A^*(p_t,\x_j)=\ah_j$, or $A^*(p_t,\x_i)=\ah_i$ and $A^*(p_t,\x_j)\neq\ah_j$.
\item\label{Enum:Case2} $A^*(p_t,\x_i)\neq\ah_i$ and $A^*(p_t,\x_j)\neq\ah_j$;
\item\label{Enum:Case3} $A^*(p_t,\x_i)=\ah_i$ and $A^*(p_t,\x_j)=\ah_j$;
\end{enumerate}

We consider Situation~\ref{Enum:Case1} first. From the proof of Lemma~\ref{Lemma:Supermartingale},
\begin{align*}
\lefteqn{\EE{\eta_{t+1}\mid\F_t,x_{t+1}\in\set{\x_i,\x_j}}}\\
  &\qquad\leq1-\frac{\eps}{2}\left[1-\frac{1}{2}\sum_{\vec{h}\in\H}p_t(\vec{h})\big(h(\x_i,\ah_i)+h(\x_j,\ah_j)\big)\right],
\end{align*}
where we explicitly replaced the definition of $\delta(a)$. If $A^*(p_t,\x_i)=\ah_i$ and $A^*(p_t,\x_j)\neq\ah_j$ (the alternative is treated similarly), we have that
\begin{align*}
\sum_{\vec{h}\in\H}p_t(\vec{h})h(\x_i,\ah_i)&\leq1 && \text{and} & \sum_{\vec{h}\in\H}p_t(\vec{h})h(\x_i,\ah_i)&\leq0,
\end{align*}
yielding
\begin{displaymath}
\EE{\eta_{t+1}\mid\F_t,x_{t+1}\in\set{\x_i,\x_j}}\leq1-\frac{\eps}{4}.
\end{displaymath}

Considering Situation~\ref{Enum:Case2}, we again have 
\begin{align*}
\lefteqn{\EE{\eta_{t+1}\mid\F_t,x_{t+1}\in\set{\x_i,\x_j}}}\\
  &\qquad\leq1-\frac{\eps}{2}\left[1-\frac{1}{2}\sum_{\vec{h}\in\H}p_t(\vec{h})\big(h(\x_i,\ah_i)+h(\x_j,\ah_j)\big)\right],
\end{align*}
where, now,
\begin{align*}
\sum_{\vec{h}\in\H}p_t(\vec{h})h(\x_i,\ah_i)&\leq0 && \text{and} & \sum_{\vec{h}\in\H}p_t(\vec{h})h(\x_i,\ah_i)&\leq0.
\end{align*}
This immediately implies
\begin{displaymath}
\EE{\eta_{t+1}\mid\F_t,x_{t+1}\in\set{\x_i,\x_j}}\leq1-\frac{\eps}{2}.
\end{displaymath}

Finally, concerning Situation~\ref{Enum:Case3}, $\vec{h}_0=\vec{h}^*$. Since $\X_i$ and $\X_j$ are 1-neighbors, $h(\x_i,\ah_i)=h(\x_j,\ah_i)$ for all hypothesis other than $\vec{h}^*$. Equivalently, $h(\x_i,\ah_i)=-h(\x_j,\ah_j)$ for all hypothesis other than $\vec{h}^*$. This implies that
\begin{displaymath}
\EE{\eta_{t+1}\mid\F_t,x_{t+1}\in\set{\x_i,\x_j}}\leq1-\frac{\eps}{2}(1-p_t(\vec{h}^*)).
\end{displaymath}

Putting everything together,
\begin{align*}
\lefteqn{\EE{\eta_{t+1}\mid\F_t}\leq}&\\
&\max\set{1-\frac{\eps}{4},1-\frac{\eps}{2}(1-p_t(\vec{h}^*)),1-\frac{\eps}{2}(1-c^*)}
\end{align*}
and
\begin{align*}
\EE{\frac{C_\tau}{C_{\tau-1}}\mid\F_{\tau-1}}
  &\leq\frac{\EE{\eta_{t+1}\mid\F_t}-p_t(\vec{h}^*)}{1-p_t(\vec{h}^*)}\\
  &\leq1-\min\set{\frac{\eps}{4},\frac{\eps}{2}(1-c^*)}.
\end{align*}
The proof is complete. \qed


\subsection{Proof of Theorem~\ref{Theo:Consistency-general}}%
\label{Proof:Consistency-general}

Let $\eps=1-\hat{c}$ and 
\begin{displaymath}
C_t=\frac{\eps-p_t(\vec{h}^*)}{p_t(\vec{h}^*)}.
\end{displaymath}
Let $a$ denote an arbitrary action in $\A_{\hat{c}}(p_t, \x_i)$, for some $\x_i,i=1,\ldots,N$. Then
\begin{align*}
\lefteqn{\PP{h^*(\x_i,a)=-1}}&\\
  &=\PP{\sum_{\vec{h}\neq\vec{h}^*}p_t(\vec{h})h(\x_i,a)>\hat{c}+p_t(\vec{h})}\\
  &\leq\PP{\sum_{\vec{h}\neq\vec{h}^*}p_t(\vec{h})>\hat{c}+p_t(\vec{h})}\\
  &=\PP{1-p_t(\vec{h}^*)>\hat{c}+p_t(\vec{h})}\\
  &=\PP{C_t>1}\\
  &\leq\EE{C_t},
\end{align*}
where, again, the last inequality follows from the Markov inequality. We can now replicate the steps in the proof of Theorem~\ref{Theo:Consistency} in Appendix~\ref{Proof:Consistency} to establish the desired result, for which we need only to prove that
\begin{displaymath}
\EE{C_{t+1}\mid\F_t}\leq C_t.
\end{displaymath}
From Lemma~\ref{Lemma:Supermartingale}, the result follows.\qed


\section*{Acknowledgements}

This work was partially supported by the Portuguese Fundação para a Ciência e a Tecnologia (INESC-ID multiannual funding) under project PEst-OE/EEI/LA0021/2011. Manuel Lopes is with the Flowers Team, a joint INRIA ENSTA-Paristech lab.


\end{document}